\documentclass[research, 10pt]{fcs}

\usepackage{amsthm}
\usepackage{amsmath}
\usepackage{amsfonts}
\usepackage{cleveref}
\usepackage{algorithm}
\usepackage{algorithmicx}
\usepackage{algpseudocode}
\usepackage{mathtools}

\usepackage{multirow}
\usepackage{array}
\usepackage{booktabs}
\usepackage{hhline}

\usepackage{adjustbox}
\usepackage{makecell}
\usepackage{enumitem}
\usepackage{booktabs}
\usepackage{verbatim}
\usepackage{mathrsfs}

\usepackage{graphicx}

\usepackage{textcomp}
\usepackage{stfloats}
\usepackage{url}

\usepackage{pifont}

\newtheorem{prop}{Proposition}

\DeclareMathOperator*{\argmax}{arg\,max}

\usepackage{xcolor}

\newcommand{\thickhline}{%
    \noalign {\ifnum 0=`}\fi \hrule height 1pt
    \futurelet \reserved@a \@xhline
}

\title{Enhancing Job Salary Prediction with Disentangled Composition Effect Modeling: A Neural Prototyping Approach}

\author[1]{Yang Ji}
\author*[1]{Ying Sun}
\author[2]{Hengshu Zhu}

\address[1]{Thrust of Artificial Intelligence, The Hong Kong University of Science and Technology (Guangzhou), Guangzhou 511458, Guangdong, China}
\address[2]{Computer Network Information Center, Chinese Academy of Sciences, Beijing 100083, China}

\fcssetup{
  received       = {April 8th, 2025},
  accepted       = {month dd, yyyy},
  corr-email     = {yings@hkust-gz.edu.cn},
}

\begin{abstract}
In the era of the knowledge economy, understanding how job skills influence salary is crucial for promoting recruitment with competitive salary systems and aligned salary expectations. Despite efforts on salary prediction based on job positions and talent demographics, there still lacks methods to effectively discern the set-structured skills' intricate composition effect on job salary. While recent advances in neural networks have significantly improved accurate set-based quantitative modeling, their lack of explainability hinders obtaining insights into the skills' composition effects. Indeed, model explanation for set data is challenging due to the combinatorial nature, rich semantics, and unique format. To this end, in this paper, we propose a novel intrinsically explainable set-based neural prototyping approach, namely \textbf{LGDESetNet}, for explainable salary prediction that can reveal disentangled skill sets that impact salary from both local and global perspectives. Specifically, we propose a skill graph-enhanced disentangled discrete subset selection layer to identify multi-faceted influential input subsets with varied semantics. Furthermore, we propose a set-oriented prototype learning method to extract globally influential prototypical sets. The resulting output is transparently derived from the semantic interplay between these input subsets and global prototypes. Extensive experiments on four real-world datasets demonstrate that our method achieves superior performance than state-of-the-art baselines in salary prediction while providing explainable insights into salary-influencing patterns.
\end{abstract}


\begin{document}

\section{Introduction}
With the rise of the knowledge economy, the skills that individuals possess play a crucial role in determining job compensation and salary levels~\cite{ng2014conservation, dix2015trade, burstein2017international}. Understanding the salary-influence of skills is vital for effective recruitment~\cite{sun2021market, ternikov2022soft, lovaglio2018skills}. It helps employers develop competitive salary systems to attract a diverse talent pool. Simultaneously, it guides job seekers align their salary expectations to find suitable jobs. However, despite significant efforts to predict salaries based on job positions~\cite{meng2018intelligent,meng2022fine} and talent demographics~\cite{lazar2004income}, there remains a shortage of effective methods to accurately model salary-influencing patterns based on skill compositions.

Indeed, the market's numerous skills create a vast combination space, with different jobs requiring unique skill sets. Analyzing their salary-influencing patterns, given the supply-demand relationship among various roles, is a complex set modeling problem. 
In recent years, neural networks have advanced set-based modeling by capturing intricate set-outcome relationships.
Existing progress, represented by DeepSets~\cite{zaheer2017deep} and Set Transformer~\cite{lee2019set}, typically learn element embeddings and employ set-pooling layers to extract permutation-invariant set representations in an end-to-end way. This brings high expressiveness in capturing element interactions and remarkable prediction accuracy. However, the lack of intrinsic explainability limits their ability to offer clear insights into the effects of skill compositions.

As a data structure with unique characteristics and internal collaboration mechanisms, sets pose unique challenges to model explanations.
An illustrative example is provided in \Cref{fig:motivation}. Specifically,
\textbf{(A)} The collective effect of set elements, rather than individual impacts, significantly influences outcomes.
For instance, as AI developers, talents possessing a comprehensive skill package tend to get higher salaries. 
Since a single skill cannot solely justify the compensation, common explanation paradigms~\cite{sundararajan2017axiomatic, sun2021market} focusing on individual input impacts struggle to provide meaningful insights into salary determination. 
Symbolizing and quantifying the collective semantic influences within a set presents a notable challenge.
\textbf{(B)} A set can be a combination of multi-faceted semantics from different interactions among skills. For instance, a front-end development position in an AI company may also require some AI knowledge, leading to a higher salary than a typical front-end role but not as high as an AI specialist.
Intuitively, the salary is jointly determined by the reflected extent and the overall market value of these global semantics.
However, disentangling the global semantics from various mixtures and modeling their interaction to influence the entire set's outcome is a challenging task.
\textbf{(C)} 
Unlike structured grid data like images, sets are discrete, conceptual, and unordered.
These high-level features arise from semantic interactions among elements, rather than spatial arrangements. Thus, existing methods~\cite{nauta2021neural, chen2019looks, li2018deep}, which rely on continuous representations for spatial visualization, face limitations with sets. Providing quantitative explanations for unique set data remains unsolved.

\begin{figure}
  \centering
  \includegraphics[width=\linewidth]{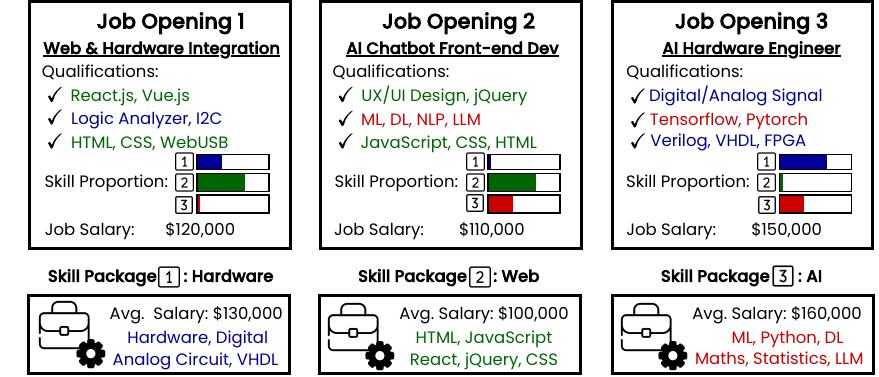}
  \caption{An illustrative example on skill-salary influence.}
  \vspace{-3mm}
  \label{fig:motivation}
\end{figure}

To tackle these challenges, we propose a novel, intrinsically explainable set-based neural prototyping approach, namely \textbf{L}ocal-\textbf{G}lobal Subset \textbf{D}is\textbf{E}ntangling \textbf{Set} \textbf{Net}work, called \textbf{LGDESetNet}, which models outcomes based on learning multi-faceted input subsets and influential prototypical sets.
With LGDESetNet, we realize explainable and accurate salary prediction while disentangling salary-influential skill sets from both an instance-based local view and a task-based global view. In particular, to learn symbolic set influence, LGDESetNet models discrete and permutation-invariant set-formed patterns in a differentiable end-to-end manner. 
Specifically, we propose a disentangled discrete subset selection layer to identify local subsets reflecting different semantics. In particular, an element co-occurrence graph density regularization is designed for enhancing dense and disentangled extraction. Next, we propose a set-oriented discrete prototype learning method to extract globally influential sets, which generates outputs in a transparent way based on the semantic interactions between the extracted subsets and global prototypes. To flexibly process external job contexts (e.g., time, city), we design two optional fusion layers to integrate skill-wise and set-wise side information in an explainable way.
Finally, we conduct extensive experiments on four real-world job salary datasets. 
The experimental results clearly show that LGDESetNet outperforms state-of-the-art self-explainable models and salary prediction models. Quantitative and qualitative analyses underscore our method's superior explainability ability to identify influential skill sets in the labor market and quantify their impacts on salary.

Our key contributions are summarized as follows:
1) We provide a novel neural framework to effectively disentangle and quantify skill sets' influence on salary from job posting data.
2) To the best of our knowledge, we are the first to develop a self-explainable deep set model that provides outcome-influential set-viewed explanations.
3) We propose a novel graph-enhanced subset selection layer and set-oriented prototypical learning method for distilling local and global influential skill sets.
4) Extensive experiments on real-world datasets, including a user study, present our model's superior salary prediction performance alongside novel explainability capability.

\section{Related Work}
This paper's related work falls into three categories: data-driven salary prediction, set modeling, and neural network explainability. We will discuss each in detail and highlight the differences between our work and existing studies.

\subsection{Data-driven Salary Prediction}
In recent years, growing efforts have been made to explore salary-influential factors from a data-driven perspective~\cite{quan2022salary}. Despite examining various factors~\cite{autor2013growth, dix2015trade, burstein2017international, sun2025market}, quantitative salary prediction remains challenging.
The evolution of machine learning offers promising solutions for dynamic and effective salary prediction.
For example, \textit{Lazar et al.}~\cite{lazar2004income} adopted the classical Support Vector Machine (SVM) on talent demographics data to predict incomes.
\textit{Meng et al.}~\cite{meng2018intelligent} proposed a holistic matrix factorization approach for salary benchmarking based on company and job position information. \textit{Meng et al.}~\cite{meng2022fine} further designed a nonparametric Dirichlet process to capture the latent skill distributions.
In~\cite{sun2021market}, a cooperative composition deep model is introduced for salary prediction by market-aware skill valuation.
\textit{Fang et al.}~\cite{fang2023recruitpro} developed a pre-training method for recruitment tasks, achieving BERT-comparable performance in the downstream salary prediction task.
Unlike existing works, our study focuses on the skill composition effect on job salary and proposes a self-explainable deep set model to provide explainable insights into skill-based salary-influencing patterns. Moreover, we holistically consider both local and global job contexts' influences (e.g, skill proficiency, base city) on skill-salary prediction.

\subsection{Set Modeling}
In data mining, frequent pattern mining~\cite{li2022frequent} is used to discover frequently appearing sets within datasets. 
Neural network-based methods face challenges with set-structured inputs due to their unordered nature.
DeepSets~\cite{zaheer2017deep} is the prominent work, which proposed a unified framework based on the set-pooling method to compress any unordered set into a single embedding vector. Based on this framework, \cite{murphy2018janossy} designed a learnable Janossy pooling layer to enhance the model flexibility.
Another research line investigated the attention mechanism to discover the interactions between set elements~\cite{lee2019set, maron2020learning, hirsch2021trees}. 
SetNorm~\cite{zhang2022set} further proposed a novel set normalization to increase model depths.
Some domain-specific studies can be also viewed as special cases of learning sets, including point cloud~\cite{qi2017pointnet, xu2018spidercnn, wu2024point, chen2024pointgpt}, graphs~\cite{maron2018invariant, kim2022pure}, hypergraph~\cite{chien2021you} and images~\cite{yao2024weakly, gordon2019convolutional}. 
However, few studies have focused on explaining the set-based modeling process.
The closest work to ours is Set-Tree~\cite{hirsch2021trees}, which provides instance-wise explanations by measuring the frequency each item occurs in the attention-sets. In contrast, our work provides both multi-faceted local explanations and global prototypical examples with explicit quantitative importance weights.

\subsection{Neural Network Explainability}

Neural network explainability aims to uncover important local or global data patterns for model predictions.
One research line is post-hoc explainability, which employs additional models like gradient-based~\cite{selvaraju2017grad, ancona2017towards} and perturbation-based~\cite{lundberg2017unified} methods to explain black-box models. However, such post-hoc methods fail to provide intrinsic model insights in many cases~\cite{rudin2019stop, schroder2023post, sundararajan2017axiomatic}.
Another approach is building self-explainable models, leveraging regularization~\cite{lemhadri2021lassonet} or attention mechanisms~\cite{miao2022interpretable} to identify critical input patterns.
Moreover, some explored prototype learning~\cite{chen2019looks} or case-based reasoning~\cite{kolodner1992introduction}, mimicking human problem-solving by matching representative examples~\cite{van2010example}.
PrototypeDNN \cite{li2018deep} and ProtoPNet \cite{chen2019looks}are pioneering works in prototype-based image classification. Subsequent work \cite{ma2024looks, wang2021interpretable, rymarczyk2021protopshare, pach2025lucidppn} aims to make explanations more concise and understandable.
Prototype learning has been applied to sequences~\cite{ming2019interpretable, hong2020interpretable, rajagopal2021selfexplain}, graphs~\cite{zhang2022protgnn, seo2024interpretable}, imitation learning~\cite{jiang2023fedskill}, traffic flow~\cite{fauvel2023lightweight}, and reinforcement learning~\cite{sun2025market, kenny2022towards}.
As aforementioned, explaining the set-based modeling process remains underexplored. This work focuses on processing set-structured data by disentangling and prototyping relationships among set entities.

\begin{table}[t]
 	\caption{Major notations in this paper.}
	\label{tab:notation}
	\begin{tabular}{lp{6.5cm}}
		\hline
		Symbol & Description\\
		\hline
		$X$ & The input skill set.\\
            $x_i$ & The $i$-$th$ skill in the skill set $X$. \\
            $lv_i$ & The skill-wise external information of the skill $x_i$. \\
            $C$ & The set-wise external information. \\
            $Y$ & The salary outcome. \\
            $U$ & The universal set of all the appeared skills. \\
            $N_{total}$ & The total number of all appeared skills. \\
            $E_{total}$ & The embedding table of all appeared skills. \\
            $\phi$ & The skill encoding function. \\
            $H$   & The number of views. \\
            $\alpha^{\text{lv}}_i$ &  The skill-wise importance weight of the skill $x_i$. \\
            $S_h$ & The $h$-th extracted subset. \\
            $E^\text{s}_h$ & The embeddings of the $h$-th subset. \\
            $G_{skill}$ & The skill co-occurrence graph. \\
            $P_i$ & The $i$-$th$ prototype. \\
            $M$ & The number of prototypes. \\
            $\gamma^\text{s}_i$ & The multi-hot selection vector of $P_i$. \\
            $\gamma^\text{lv}_i$ & The skill-wise weights of $P_i$. \\
            $\gamma^\text{sal}_i$ & The salary effect weight of $P_i$. \\
            $E^\text{p}_i$ & The continuous embeddings of $P_i$. \\
            $\mathcal{T}$ & The transformation layer. \\
            $Z^\text{s}_h$ & The projected embeddings of $S_h$. \\
            $Z^\text{p}_i$ & The projected embeddings of $P_i$. \\
		\hline
	\end{tabular}
	\vspace{-2mm}
\end{table}

\section{Problem Formulation}
\label{sec:data-description}

We utilize real-world job posting data~\cite{JobData,JobData2}.
Each job posting consists of the job salary range and descriptive external information, such as company, city, time, and skill level requirements. To mine the interaction among skill compositions, we formulate each job posting as a tuple $(X, C, Y)$, where $X$ represents a skill set, $C$ represents set-wise external information (e.g., job contexts such as time, city), and $Y$ represents its salary outcome. In particular, $X$ is an unordered set of skills $X = \{x_1, x_2, ..., x_n\}$, where $n$ indicates the number of skills. Each skill $x_i$ is drawn from a universal set $U$. $N_{total}=|U|$ indicates the total number of skills across the dataset. In many real-world scenarios, each skill $x_i$ is possibly associated with skill-wise external information $lv_i$ (e.g., proficiency level). 

Formally, we define the task of this paper as a skill-based explainable salary prediction problem. Using training data, it aims to train a neural network model $f$ that can predict salary outcomes based on skill sets with associated external information and provides explainable insights into the reasoning process. Additionally, $f$ should be permutation invariant for set processing, implying that $f(X) = f(\pi(X))$ for any permutation $\pi$ of the indices ${1, 2, \ldots, |X|}$ and all $X$ in the input domain. \Cref{tab:notation} summarizes the major notations.

\section{Methodology}
In this section, we first present the overview of LGDESetNet. Then, we introduce the details of each module and our training procedure for efficient set-oriented prototype learning.

\begin{figure*}[t]
    \begin{center}
        \includegraphics[width=\linewidth]{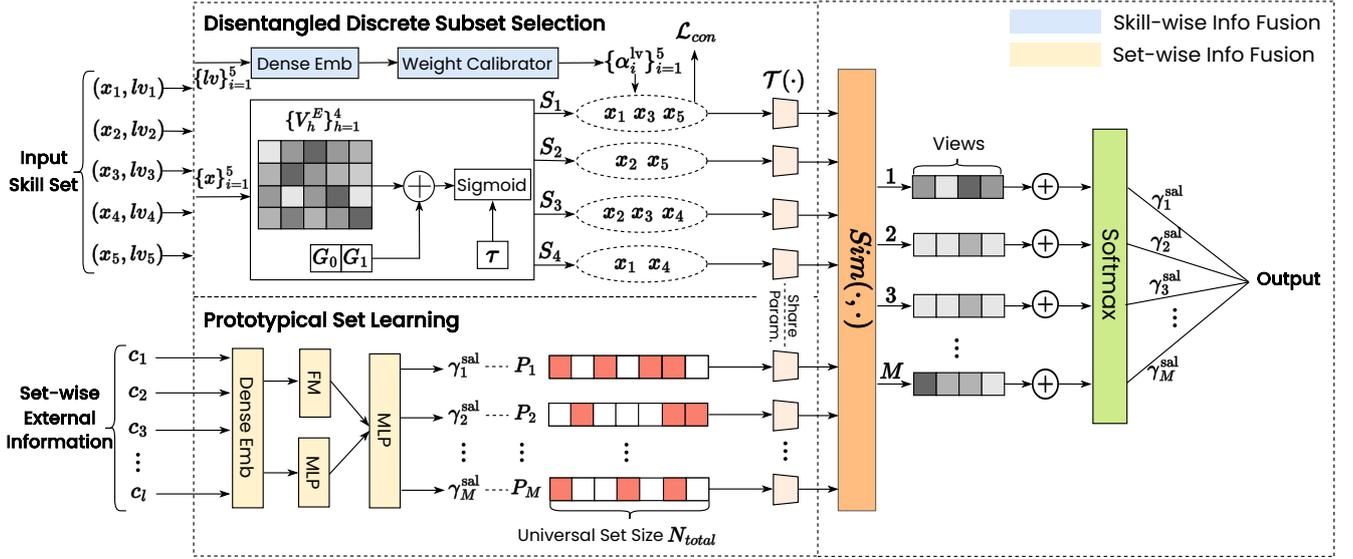}
        \caption{The architecture of LGDESetNet.}
        \label{fig:model}
    \end{center}
\end{figure*}

\subsection{Framework Overview}
As illustrated in~\Cref{fig:model}, the LGDESetNet model consists of two main components: disentangled discrete subset selection and prototypical set learning.
Our model aims to provide explainable salary prediction by disentangling and quantifying the influence of skill sets on salaries from both local and global perspectives.
To symbolize the collective semantic factors, LGDESetNet learns $M$ global representative skill compositions as prototypes $\{P_1, P_2, .., P_M\}$, each represented as a set of skills.
Each prototype has specific semantics reflected by a cluster of influencing skill subsets on the market and has independent effects on salaries varying across different job contexts.
Given an input skill set, which can reflect multiple views of skill compositions, LGDESetNet identifies $H$ key subsets $\{S_1, S_2, ..., S_H\}$, representing diverse semantic views.
These subsets are subsequently compared with global prototypes to determine pairwise similarity scores. 
The overall match of each prototype to the input is weighted and combined into the salary prediction.

To achieve this, the disentangled discrete subset module introduces a differentiable multi-view subset selection network. Through the Gumbel attention mechanism, it can explicitly represent each extracted subset as a discrete set instead of continuous embeddings. 
We further introduce a weight calibrator to measure the importance of each skill within the input set based on its external information.
A co-occurrence graph density-based regularizer is proposed to encourage extracting semantically meaningful skill subsets.
To symbolize global semantics, the prototypical set learning module employs a set-oriented prototypical part layer with an efficient embedding-projection training algorithm. This strategy enables learning discrete skill sets that align the distribution of local semantics of key subsets recognized from the inputs. We also introduce a contextual salary influence module to dynamically capture the varied salary effects of each prototype under different job contexts. Notably, the skill-wise and set-wise fusion layers are optional and can be flexibly integrated based on external job contexts.

\subsection{Disentangled Discrete Subset Selection}
\label{model:MSSN}
The disentangled discrete subset selection component comprises a differentiable multi-view subset selection network, a skill-wise weight calibrator, and co-occurrence graph density regularization.

\subsubsection{Multi-view Subset Selection Network}
As demonstrated in our motivating example, real-world samples can reflect multiple skill semantics that influence the outcome from different views. In this case, regular prototype-based methods that learn representative patterns from the entire input space can reduce the model's performance and obscure the model's explainability by conflating semantics. To address this limitation, LGDESetNet disentangles the input into multiple subsets as its multi-view representations.

First, LGDESetNet uses a permutation-invariant encoder to embed each set's overall semantic features. The input skill vector $V_X \in \mathbb{R}^n$ can be in any permutation order of inner skills, where $n$ denotes the size of $X$.
The representation $\phi(x_i)$ for each skill in $U$ is learned as an embedding table $E_{total} \in \mathbb{R}^{N_{total} \times d}$, where $d$ represents the embedding dimension. Next, a permutation-invariant pooling function $\text{pool}(\cdot)$ aggregates the input set's skill representations into the set representation. %
We introduce the attention mechanism to capture the interactions of skills to different semantics, the skill and set's representations are projected into $H$ distinct subspaces. 
For the $h$-th view, each row in $V_h^\text{K}, V_h^\text{V} \in \mathbb{R}^{n \times d_h}$ denotes the key and value vectors of the corresponding skill in $V_X$, and $V_h^\text{Q} \in \mathbb{R}^{d_h}$ denotes the query projected from the set embedding. Using the dot-product attention mechanism~\cite{vaswani2017attention}, we calculate the similarity of each skill in each view as: 
\begin{equation}
    V^\text{E}_h = V_h^\text{Q}(V_h^\text{K})^T/\sqrt{d_h}.
\end{equation} 
Each entry $V^\text{E}_h \in \mathbb{R}^{N}$ models the interaction weight of the $i$-th skill with the $h$-th view.

While activation functions can be used to project interaction weights to the range [0, 1], the continuous interaction score hinders the network's ability to model in an explicit binary manner. To address this issue, we introduce Gumbel-Sigmoid~\cite{gumbel1948statistical, geng2020does} to enable differentiable skill selection. Specifically, we incorporate Gumbel noise with the Sigmoid function and output the selected skills for each view's subset as:
\begin{equation}
    \begin{split}
        A^\text{s}_h = \frac{\exp((V^\text{E}_h + G_0)/\tau)}{\exp((V^\text{E}_h + G_0)/\tau) + \exp((G_1)/\tau)}, \\
        u_i \sim \text{Uniform}(0, 1), i\in \{0, 1\},\\
        G_i = -\log(-\log(u_i)), i\in \{0, 1\},
    \end{split}
\end{equation}
where $G_0$ and $G_1$ are two i.i.d Gumbel noises, the temperature $\tau$ controls the smoothness of the sampling function. $A^\text{s}_h$ becomes a multi-hot vector for selecting a subset $S_h$ when $\tau$ approaches $0$ and becomes a continuous vector as $\tau \rightarrow \infty$. In the training loop, we gradually reduce $\tau$ to learn the multi-hot representation of the selected subset.

\subsubsection{Skill-wise External Information Fusion}
In practice, skills are usually associated with external information, offering nuanced insights into its impact on salary determination.
For instance, while ``Algorithms (Understand)'' and ``Algorithms (Specialist)'' share identical skill semantics, the latter likely exerts a greater influence within the skill set.
This motivates us to design a weight calibrator for explicitly modeling these skill-wise dynamic relationships.
We adjust the weights of each skill's embeddings to discern the varied importance of skills within a set while preserving original skill semantics.
Specifically, we learn the dense embeddings from the external information $lv_i$ of the skill $x_i$ and project it into a weight value $\alpha^{\text{lv}}_i$ ranging in $[0, 1]$ with the Sigmoid function.
In this way, we can get fused embedding $E^{\text{s}}_{h}$ of the extracted subset $S_h$ by integrating the skill embeddings weighted by their skill-wise importance before the pooling operation:
\begin{equation}
    \begin{split}
        E^{\text{s}}_{h} & = \text{pool}\left( \{ \alpha^{\text{lv}}_ i  \phi(x_i) | x_i \in S_h \}\right).
    \end{split}
\end{equation}

\begin{figure}[t]
    \centering
    \begin{minipage}[t]{0.45\linewidth}
        \centering
        \includegraphics[width=\linewidth]{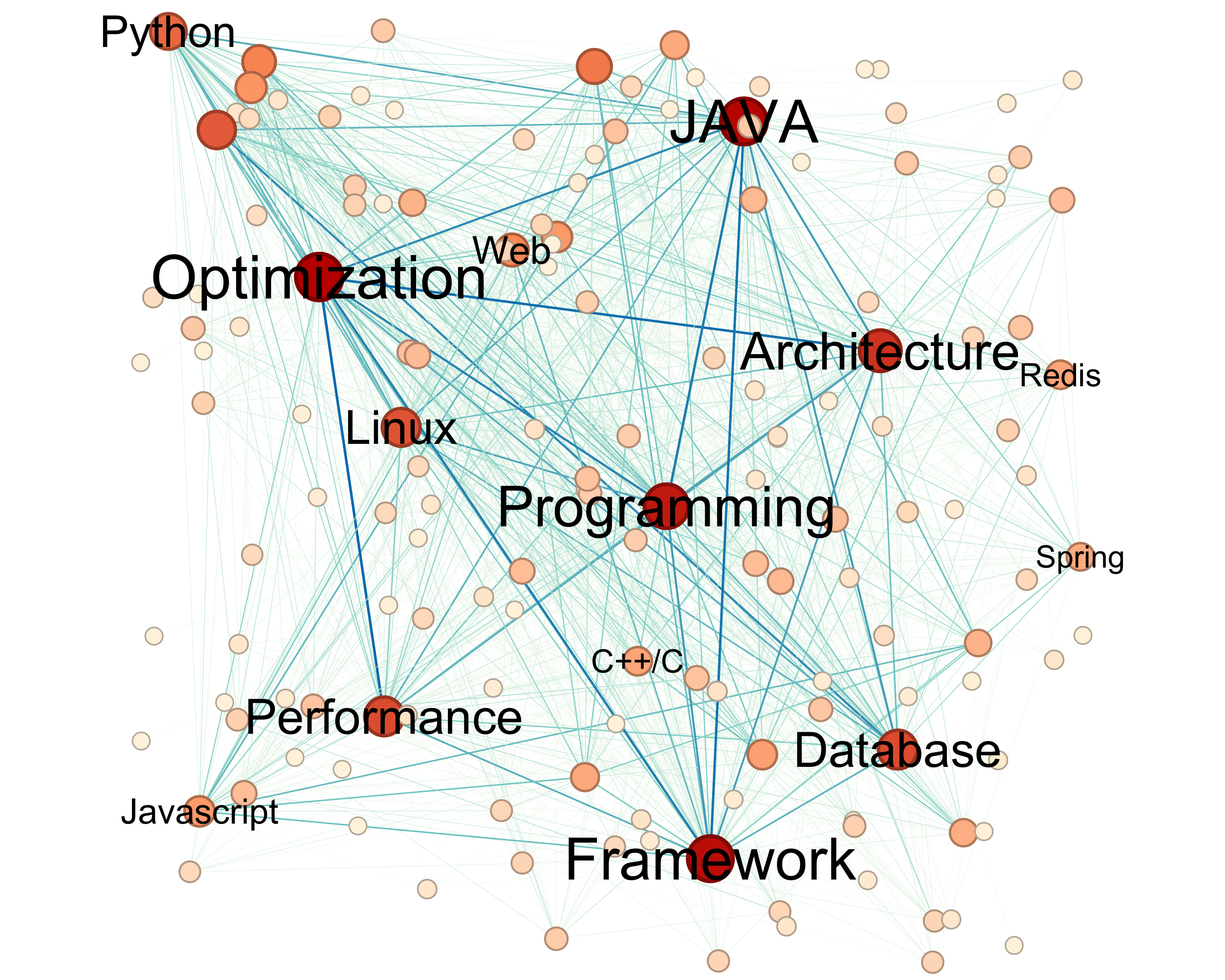}
        \caption*{(a) IT}
    \end{minipage}
    \hspace{0.03\linewidth}
    \begin{minipage}[t]{0.45\linewidth}
        \centering
        \includegraphics[width=\linewidth]{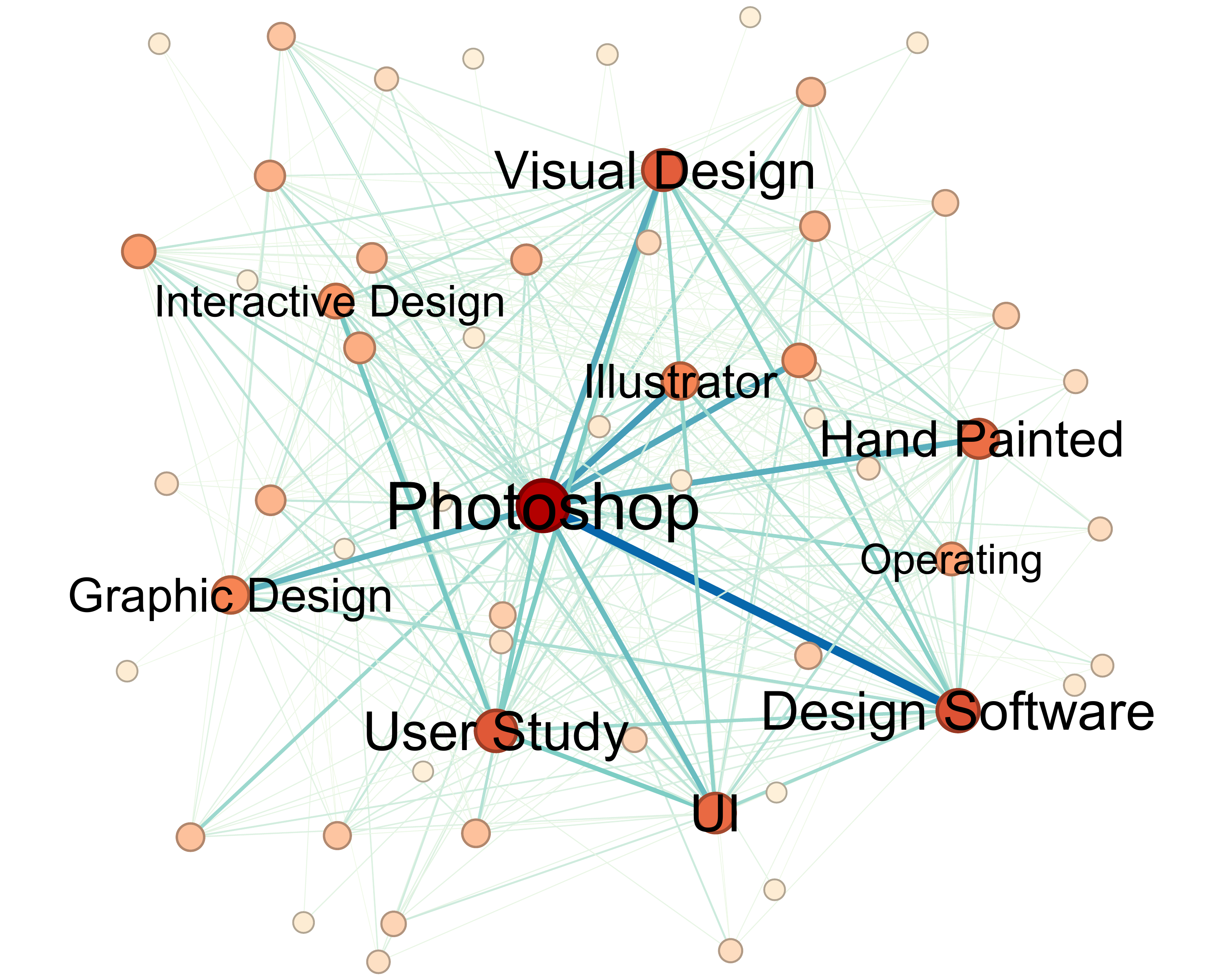}
        \caption*{(b) Designer}
    \end{minipage}
    \caption{Illustrative examples of co-occurrence graphs. Low-weight edges are eliminated to ease presentation.}
    \label{fig:graph}
\end{figure}

\subsubsection{Co-occurrence Graph Density Regularization}
To enhance explainability and robustness, extracted skill subsets should reflect meaningful real-world interactions.
Yet, focusing solely on prediction performance could generate subsets cluttered with irrelevant skills, reducing semantic clarity. 
To address this, we introduce a regularization approach aimed at refining these subsets to meet key intuitive criteria: (i) Semantic coherence within subsets, (ii) Distinct semantic boundaries between different subsets, and (iii) Avoidance of trivial subset selections (neither fully inclusive nor empty skill sets).

Observing that skills frequently appeared together in job postings likely share semantic connections (e.g., ``Machine Learning'', ``Algorithms'', and ``Python'' often cluster in AI roles), we leverage global semantic interactions inferred from such co-occurrence patterns. Direct co-occurrence analysis, however, is computationally expensive and hampered by data sparsity. %
Therefore, we model these relationships using a skill co-occurrence graph $G_{skill} = \langle U, L \rangle$, drawing from skill pair associations across the dataset, where $U$ is the universal set of skills and $L$ represents the links indicating co-occurrence frequencies between skill pairs.
As exemplified in~\Cref{fig:graph}, each graph node corresponds to a skill, and a link \(l_{i,j}\) between nodes \(x_i\) and \(x_j\) signifies their connectivity, quantified by their co-occurrence frequency \(w_{l_{i, j}}\).
To quantify semantic connections within subsets, we design a graph density-based~\cite{khuller2009finding} regularization, focusing on the subgraph $G_h=\langle S_h, L_{G_h} \rangle$ for each subset $S_h$. The subgraph's density, calculated as $\frac{\sum_{l \in L_{\mathrm{G}_h}} w_l}{|L_{{\mathrm{G}_h}|}}$, gauges the semantic cohesiveness of the subset. By maximizing this density, we encourage subsets that are semantically well-connected, while minimizing trivial or overly inclusive selections. Our graph regularization of the extracted subsets $\{S_h\}_{h=1}^H$ is thus defined as:
\begin{equation}
    \mathcal{L}_{con} = - \log \frac{1}{H} \sum_{h=1}^H \frac{\sum_{l \in L_{\mathrm{G}_h}} w_l}{|L_{\mathrm{G}_h}|}.
\end{equation}
This approach not only ensures subsets are semantically dense but also supports the network's focus on enhancing performance without losing critical information through subset disentanglement.

\subsection{Prototypical Set Learning}
Based on extracted subsets, LGDESetNet matches the input with learnable prototypes representing global semantics and accordingly models the output as an aggregation of the global contributions of these prototypes.

\subsubsection{Set-oriented Prototypical Part Layer}
The prototypical part network learns $M$ input-independent prototypical skill sets $\{P_i\}_{i=1}^M$. 
To provide an explicit skill set semantic, we adopt an approach distinct from conventional embedding-based prototype learning.
Specifically, we utilize a discrete multi-hot vector to represent the prototypes, denoted as $\gamma^\text{s} \in \{0, 1\}^{M \times N_{\text{total}}}$. Each row $\gamma^\text{s}_i \in \{0, 1\}^{N_{\text{total}}}$ indicates a prototype’s multi-hot vector and its inner entries show if a skill is contained. 
Considering a global semantic may involve skills with different importance in practice, a weight vector $\gamma^\text{lv}_i \in [0,1]^{N_{\text{total}}}$ is associated with each prototype, indicating each skill's importance within the prototype.
The differentiable training procedure of set-oriented prototypes will be detailed in~\Cref{model:training}.

We assess the alignment between the input set and each prototype (i.e., global semantics) by comparing the similarity of extracted subsets to the prototypes.
Notably, the discrete set representation brings the curse of dimensionality, complicating set similarity estimation. 
For instance, ``C'' and ``C++'' are treated as distinct skills despite their semantic closeness, potentially hindering model convergence and accuracy.
To address this, we project prototype $P_i$ and subset $S_h$ into semantic representation vectors with the same skill embedding table as the input set, denoted as $E^\text{p}_i = \text{pool}\left( \{ \gamma^{\text{s}}_ i \gamma^{\text{lv}}_ i \phi(x_i) | x_i \in U \}  \right)$ and $E^{\text{s}}_{h} = \text{pool}\left( \{ \alpha^{\text{lv}}_ i  \phi(x_i) | x_i \in S_h \}\right)$. 
Considering the distribution shift between the whole set and single-view subsets, we introduce a transformation layer $\mathcal{T}$ to further project the subsets and prototypes into a subset representation space as $Z^\text{p}_i = \mathcal{T}\left(E^\text{p}_i\right)$ and $Z^\text{s}_h = \mathcal{T}\left(E^{\text{s}}_{h}\right)$.
In this way, we better capture semantic-level associations while ensuring the similarity can be measured by aligned representation of prototypes and extracted subsets.

Finally, we calculate the pairwise similarity between a subset $S_h$ and a prototype $P_i$ as
\begin{equation}
    sim(S_h, P_i) = \text{log}\left(\frac{\|Z^\text{s}_h - Z^\text{p}_i\|_2^2 + 1}{\|Z^\text{s}_h - Z^\text{p}_i\|_2^2 + \epsilon}\right).
\end{equation}
For each prototype, we aggregate their similarity to different extracted subsets and apply softmax to obtain a global normalized similarity score.
Moreover, each prototype is associated with a job-context-dependent weight value $\gamma^\text{sal}_i$, reflecting its salary effects in specific job contexts. We constrain these weight values as non-negative values to enhance explainability.
Thus, the output can be modeled as the weighted average of the similarity scores between each prototype and the input skill set:%
\begin{equation}
    y = \sum_{i=1}^M \gamma^{\text{sal}}_i \cdot  \text{softmax}(\sum_{h=1}^H sim(S_h, P_i)).
\end{equation}

\subsubsection{Set-wise External Information Fusion}
In practice, each skill set is enriched with set-wise external information, which is also useful in determining salary outcomes.
For example, job salaries are highly correlated with the work experience and the pricing level of the base city. 
However, directly fusing such information into deep and wide embeddings~\cite{cheng2016wide} could obscure the model's explainability. 
Therefore, we separate the set-wise external information from skill embeddings, acknowledging that the salary effects of global prototypical skill sets may vary across different job scenarios. 
Specifically, we model this prototype salary effect dynamic as a function of the set-wise external features $C = \{c_1, c_2, ..., c_l\}$. 
Drawing inspiration from DeepFM~\cite{guo2017deepfm}, we utilize a factorization machine encoder for low-order feature interactions and an MLP encoder for high-order interactions. By aggregating the embeddings from both encoders, we then employ another MLP to estimate the salary effects of each prototype, ensuring a dynamic and contextual assessment of skill sets' impact on salary predictions:
\begin{equation}
    \begin{split}
        \gamma^\text{sal} & = \text{MLP}\left(\text{FM}(C) + \text{MLP}(c_1|c_2|...|c_l)\right). \\
        \text{FM}(C) & = w_0 + \sum_{i=1}^l w_i c_i + \sum_{i=1}^l \sum_{j=i+1}^l  c_i \odot c_j, 
    \end{split}
\end{equation}
where $\odot$ denotes the element-wise multiplication between features.

\subsubsection{Prototype Learning Regularization}
To enhance explainability, we introduce regularization methods for our prototypical set learning procedure, whose basic ideas are widely adopted in prototype learning for other data structures~\cite{zhang2022protgnn, ming2019interpretable}.
First, since prototypes are desired to be representative exemplars for the subsets, we introduce a representation regularization term to encourage each subset to be close to at least one prototype and encourage each prototype to be close to at least one extracted subset.
\begin{equation}
    \mathcal{L}_{rep} = \frac{1}{H} \sum\limits_{h=1}^H \min_{i \in [1, M]} \|Z^\text{s}_h - Z^\text{p}_i\|_2^2 
    + \frac{1}{M} \sum_{i=1}^M \min_{h \in [1, H]} \|Z^\text{s}_h - Z^\text{p}_i\|_2^2.
\end{equation}

Moreover, to learn an extensive prototypical embedding space, we introduce a diversity regularization term to scatter prototypes. 
\begin{equation}
    \mathcal{L}_{div} = \frac{1}{M} \sum_{i=1}^M \sum_{j=i+1}^M \max(0, sim(P_i, P_j) - \theta_{min}),
\end{equation}
where $\theta_{min}$ is a pre-defined threshold that verdicts whether a prototype pair is close or not. We set $\theta_{min}$ to $1.0$ in our experiments.

To summarize, the overall loss function is formulated as:
\begin{equation}
    \mathcal{L} = \mathcal{L}_{pred} + \lambda_{con} \mathcal{L}_{con} + \lambda_{rep} \mathcal{L}_{rep} + \lambda_{div} \mathcal{L}_{div}.
\end{equation}
where $\mathcal{L}_{pred}$ is an appropriate loss function for outcome prediction and $\lambda_{*}$ are hyperparameters for controlling the impacts of regularization terms on model explainability.

\begin{algorithm}[t]
    \caption{Training Procedure of LGDESetNet}
    \label{alg::train}
    \begin{algorithmic}[1]
      \Require
        The dataset $\mathcal{D} = \{(X_i, C_i, Y_i)\}_{i=1}^{N_{d}}$, Total number of epochs $N_t$, Warm-up epoch $N_w$, Prototype projection period $\tau$, and Prototype refinement epoch $N_f$;
      \Ensure
        Parameters of the trained model $\Phi$ and set-formed prototypical concepts $\{P_k\}_{k=1}^M$;
      \State Initialize the model parameters $\Phi$;
      \State Extract the frequent skill sets $\mathcal{D}_{S}$ from $\mathcal{D}$;
      \For{Epoch $e$ = 1, 2,..., $N_t$}
         \State Fit $\Phi$ on $\mathcal{D}$ by minimizing $\mathcal{L}$;
         \If{$e > N_w$ \textbf{and} $e \% \tau == 0$}
            \State Infer all the extracted subsets $\mathcal{S}$ using $\Phi$ on $\mathcal{D}_S$;
            \For{$k = 1,2,...,M$}
                 \State $\gamma^{\text{s}*}_k, \gamma^{\text{lv}*}_k \leftarrow \argmax_{S \in \mathcal{S}} sim(S, P_k)$
                 \State $E^{p*}_k \leftarrow \text{pool}\left( \{ \gamma^{\text{s}*}_k \gamma^{\text{lv}*}_ k \phi(x_i) | x_i \in U \} \right)$
            \EndFor
         \EndIf
      \EndFor
      \For{$k = 1,2,...,M$}
          \State Initialize $\triangle E^p_k = \mathcal{T}\left( \text{pool}\left( \{ \left(\gamma^{\text{s}*}_k \gamma^{\text{lv}*}_ k + \triangle P_k \right)\phi(x_i) | x_i \in  U\}  \right)\right)$;
    \State Freeze $E^{p*}_k$ and replace $E^p_k$ with $E^{p*}_k + \triangle E^p_k$;
    \EndFor

    \For{$l = 1,2,...,N_f$}
        \State Fit $\Phi$ on $\mathcal{D}$ by minimizing $\mathcal{L} + \lambda_p \sum_{k=1}^M \|\triangle P_k\|_1$;
    \EndFor
    \State $\{P_k\}_{k=1}^M \leftarrow \{\gamma^{\text{s}*}_k + \triangle P_k, \gamma^{\text{lv}*}_ k\}_{k=1}^M$;\\
    \Return $\Phi$ and $\{P_k\}_{k=1}^M$;
    \end{algorithmic}
\end{algorithm}

\begin{figure}
    \centering
    \includegraphics[width=\linewidth]{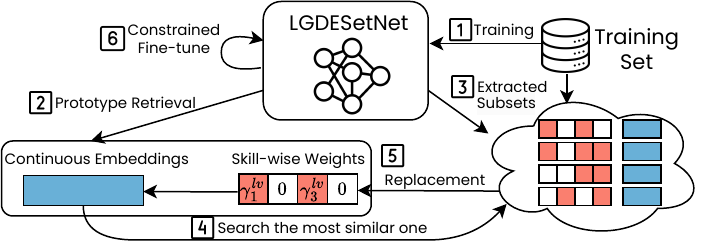}
    \caption{Training procedure of LGDESetNet.}
    \label{fig:algorithm}
\end{figure}

\subsubsection{Embedding-Projection Training Procedure}
\label{model:training}
Training discrete prototypes directly using gradient descent on randomly initialized values is challenging due to the non-contiguous nature of prototype representation.
Direct L1 regularization could incur stringent constraints and adversely affect the model performance due to the skill sparsity.
To address these limitations, we introduce an embedding-projection iterative learning algorithm for discrete prototypes. \Cref{alg::train} details the training procedure with \Cref{fig:algorithm} graphically illustrating the workflow. The algorithm proceeds in five steps:
(1) We start by optimizing the target embeddings $E^\text{p}$ for prototypes, updating them without back-propagating gradients to the discrete representations. Since $E^\text{p}$ is contiguous, gradient descent can effectively optimize it. This stage includes $N_w$ warm-up epochs to guide $E^\text{p}$ towards the appropriate semantic space. 
(2) Once retrieving the continuous prototype embeddings from the trained model, (3) we identify the closest matching subsets from frequent skill sets in the training data, (4) and then update the discrete prototypes and their corresponding embeddings $E^{\text{p}*}$. (5) Upon finalizing, we enhance the prototype's adaptability by adding a minimal bias term to its discrete representation. Specifically, we freeze $E^{\text{p}*}$ and introduce a learnable bias term $\triangle E^\text{p}$, whose weight vector $\triangle P$ is initialized as zeros and L1-regularized. This fine-tuning step preserves explainability while allowing for nuanced expression of prototypical semantic uncertainty. To demonstrate the effectiveness of our architecture design for set-based modeling, we provide a theoretical proof of network permutation invariance in \ref{appendix:proof}.

\begin{table}[t]
    \centering
    \caption{Detailed Dataset Features}
    \label{tab:dataset-feature}
    \resizebox{\linewidth}{!}{
    \begin{tabular}{l|cccc}
        \toprule
        Features & IT & Designer & High-tech & Financial \\ 
        \midrule
        \# Job Postings  & 212,676 & 18,588 & 128,346 & 45,394 \\
        \# Skills        & 1374    & 138    & 203     & 385    \\
        \midrule
        Skill Level      & \ding{51} & \ding{51} &         &         \\
        City             & \ding{51} & \ding{51} & \ding{51} & \ding{51} \\
        Company Name     & \ding{51} & \ding{51} &         &         \\
        Company Size     & \ding{51} & \ding{51} & \ding{51} & \ding{51} \\
        Company Stage    & \ding{51} & \ding{51} & \ding{51} &         \\
        Industry         & \ding{51} & \ding{51} & \ding{51} &         \\
        Work Experience  & \ding{51} & \ding{51} &         &         \\
        Job Temptation   & \ding{51} & \ding{51} &         &         \\
        Time             & \ding{51} & \ding{51} &         &         \\
        Employer Type    &          &          &         & \ding{51} \\
        \bottomrule
    \end{tabular}
    }
\end{table}

\subsection{Complexity Analysis}
\label{sec:complexity}
We present a theoretical analysis of LGDESetNet's computational complexity to demonstrate its feasibility for real-world deployment. The model's time complexity is dominated by two primary components: the disentangled discrete subset selection module with complexity $O(n^2 \times d \times H)$ for input sets of size $n$ with embedding dimension $d$ across $H$ views, and the prototypical set learning component requiring $O(H \times M \times d)$ operations for comparing extracted subsets with $M$ prototypes. Combined, the overall inference time complexity is $O(n^2 \times d \times H + H \times M \times d)$, which remains manageable for typical skill set sizes. 
The quadratic dependency on input set size $n$ aligns with attention-based methods in set processing, including Set-Tree~\cite{hirsch2021trees} and SESM~\cite{geng2020does}, and remains computationally feasible given the typically modest number of skills in real-world job postings.
Importantly, our model's linear scaling with respect to views ($H$) and prototypes ($M$) enables precise control over computational requirements and model capacity, allowing for effective deployment across various computational environments. For example, reducing $H$ decreases skill set selection costs, while limiting $M$ reduces prototype comparison operations. The skill/set-wise information fusion modules can be simplified to constant-time operations when external information is unnecessary, further optimizing computational requirements. These properties ensure LGDESetNet remains computationally efficient while preserving explainability across diverse applications.

\section{Experiments}

\begin{table*}[t]
    \centering
    \caption{{Overall performance evaluation. The listed methods include two categories: salary prediction models and self-explainable models. The best results are highlighted in \textbf{bold}.}}
    \resizebox{\linewidth}{!}{%
    \begin{tabular}{l|cccccccc}
    \toprule[0.5pt]
    \multirow{4}{*}{Methods}
    & \multicolumn{4}{c}{Skill-Salary Job Postings~\cite{sun2021market}} & \multicolumn{4}{c}{Job Salary Benchmarking~\cite{meng2022fine}} \\
    \cmidrule(r){2-5} \cmidrule(r){6-9}
    & \multicolumn{2}{c}{IT} & \multicolumn{2}{c}{Designer} & \multicolumn{2}{c}{High-tech} & \multicolumn{2}{c}{Finance} \\ \cmidrule(r){2-9}
    & RMSE ($\downarrow$) & MAE ($\downarrow$) & RMSE ($\downarrow$) & MAE ($\downarrow$) & RMSE ($\downarrow$) & MAE ($\downarrow$) & RMSE ($\downarrow$) & MAE ($\downarrow$) \\
    \midrule
    $\text{HSBMF}$ & $\text{5.291}_{\pm \text{0.017}}$ & $\text{3.939}_{\pm \text{0.015}}$ & $\text{4.612}_{\pm \text{0.025}}$ & $\text{3.371}_{\pm \text{0.021}}$ & $\text{3.844}_{\pm \text{0.016}}$ & $\text{2.872}_{\pm \text{0.013}}$ & $\text{4.299}_{\pm \text{0.025}}$ & $\text{3.150}_{\pm \text{0.019}}$ \\
    SSCN & $\text{4.762}_{\pm \text{0.063}}$ & $\text{3.484}_{\pm \text{0.052}}$ & $\text{3.841}_{\pm \text{0.143}}$ & $\text{2.766}_{\pm \text{0.102}}$ & $\text{3.663}_{\pm \text{0.162}}$ & $\text{2.740}_{\pm \text{0.106}}$ & $\text{4.037}_{\pm \text{0.134}}$ & $\text{2.948}_{\pm \text{0.096}}$ \\
    NDP-JSB & $\text{5.342}_{\pm \text{0.021}}$ & $\text{4.073}_{\pm \text{0.021}}$ & $\text{4.623}_{\pm \text{0.044}}$ & $\text{3.262}_{\pm \text{0.023}}$ & $\text{3.875}_{\pm \text{0.086}}$ & $\text{2.827}_{\pm \text{0.067}}$ & $\text{4.295}_{\pm \text{0.075}}$ & $\text{3.146}_{\pm \text{0.071}}$ \\ \midrule
    Set-Tree & $\text{5.552}_{\pm \text{0.065}}$ & $\text{4.367}_{\pm \text{0.083}}$ & $\text{4.339}_{\pm \text{0.149}}$ & $\text{3.269}_{\pm \text{0.061}}$ & $\text{3.953}_{\pm \text{0.092}}$ & $\text{2.998}_{\pm \text{0.065}}$ & $\text{4.406}_{\pm \text{0.079}}$ & $\text{3.206}_{\pm \text{0.057}}$ \\ %
    SESM & $\text{4.489}_{\pm \text{0.034}}$ & $\text{3.532}_{\pm \text{0.025}}$ & $\text{3.812}_{\pm \text{0.123}}$ & $\text{2.767}_{\pm \text{0.064}}$ & $\text{3.623}_{\pm \text{0.029}}$ & $\text{2.817}_{\pm \text{0.017}}$ & $\text{3.967}_{\pm \text{0.031}}$ & $\text{2.940}_{\pm \text{0.011}}$ \\
    ProtoPNet & $\text{4.503}_{\pm \text{0.029}}$ & $\text{3.423}_{\pm \text{0.037}}$ & $\text{3.794}_{\pm \text{0.013}}$ & $\text{2.772}_{\pm \text{0.076}}$ & $\text{3.602}_{\pm \text{0.110}}$ & $\text{2.786}_{\pm \text{0.025}}$ & $\text{4.116}_{\pm \text{0.144}}$ & $\text{3.013}_{\pm \text{0.063}}$ \\
    TesNet & $\text{4.555}_{\pm \text{0.102}}$ & $\text{3.698}_{\pm \text{0.037}}$ & $\text{3.804}_{\pm \text{0.063}}$ & $\text{2.902}_{\pm \text{0.012}}$ & $\text{3.692}_{\pm \text{0.096}}$ & $\text{2.806}_{\pm \text{0.081}}$ & $\text{4.241}_{\pm \text{0.108}}$ & $\text{3.094}_{\pm \text{0.071}}$ \\ 
    ProtoConcepts & {$\text{4.468}_{\pm \text{0.062}}$} & $\text{3.412}_{\pm \text{0.025}}$ & $\text{3.792}_{\pm \text{0.083}}$ & $\text{2.774}_{\pm \text{0.052}}$ & $\text{3.504}_{\pm \text{0.061}}$ & $\text{2.776}_{\pm \text{0.026}}$ & $\text{4.023}_{\pm \text{0.034}}$ & {$\text{2.961}_{\pm \text{0.048}}$} \\ \midrule
    LGDESetNet & \textbf{${\text{4.162}_{\pm \text{0.012}}}$} & \textbf{${\text{3.141}_{\pm \text{0.041}}}$} & \textbf{${\text{3.473}_{\pm \text{0.058}}}$} & \textbf{${\text{2.559}_{\pm \text{0.062}}}$} & \textbf{${\text{3.327}_{\pm \text{0.038}}}$} & \textbf{${\text{2.434}_{\pm \text{0.037}}}$} & \textbf{${\text{3.775}_{\pm \text{0.046}}}$} & \textbf{${\text{2.768}_{\pm \text{0.031}}}$} \\
    \bottomrule[0.5pt]
    \end{tabular}}
    \label{table:lower_performance}
    \vspace*{-2mm}
\end{table*}

\subsection{Experimental Setup}

\subsubsection{Datasets}
The data descriptions are provided in \Cref{sec:data-description}. We use four real-world, publicly available job posting datasets with distinct job and skill distributions: (1) Skill-Salary Job Postings datasets~\cite{JobData}, including IT and designer topics. (2) Job Salary Benchmarking datasets~\cite{JobData2}, covering high-tech and financial industries. We provide overview of detailed dataset features in \Cref{tab:dataset-feature}, and provide more descriptive statistics in~\ref{appendix:dataset}.

\subsubsection{Baseline Methods}
We compare our LGDESetNet with the following state-of-the-art salary prediction models:
\begin{itemize}[leftmargin=0pt]
    \item \textbf{HSBMF}~\cite{meng2018intelligent} is a salary benchmarking model utilizing a Bayesian approach to capture the hierarchical structure of massive job postings.
    \item \textbf{SSCN}~\cite{sun2021market} is an enhanced neural network with cooperative structure for separating job skills and measuring their values based on job postings.
    \item \textbf{NDP-JSB}~\cite{meng2022fine} develops a non-parametric Dirichlet process-based latent factor model to jointly model the latent representations of job contexts.
\end{itemize}

We also compare LGDESetNet with the following self-explainable models for set modeling. 
\begin{itemize}[leftmargin=0pt]
    \item \textbf{Set-Tree}~\cite{hirsch2021trees} is a tree-based general framework for processing sets and quantifies the relative importance of each item according to its frequency.
    \item \textbf{SESM}~\cite{geng2020does} is an instance-wise selection method that discovers a fixed number of sub-inputs to explain its own prediction.
    \item \textbf{ProtoPNet}~\cite{chen2019looks} is the pioneering work in prototype-based models, which dissects the input images into patches for fine-grained interpretation.
    \item \textbf{TesNet}~\cite{wang2021interpretable} is another prototype-based image recognition deep network that introduces a transparent embedding space for disentangled global prototype construction.
    \item \textbf{ProtoConcepts}~\cite{ma2024looks} modifies prototype geometry to enable visualizations of prototypes from multiple training images.
\end{itemize}
Specifically, directly applying SESM and prototype-based models is not viable due to the discrete and unordered nature of set data. To address this, we integrate the DeepSets framework and use LGDESetNet's external information fusion method for fair comparison.

\subsubsection{Implementation Details}
We evaluate the performance based on RMSE/MAE for salary prediction tasks. 
All results are reported based on 5 rounds of experiments.
The datasets are partitioned into training/validation/testing sets by a ratio of 6:2:2.
For fair comparison, we select the hyperparameter configurations based on the best performance on the validation set.
For all experiments, prototype number $M$ is set to 64 and head number $H$ is set to 4. The loss function hyperparameters $\lambda_{rep}$, $\lambda_{con}$, and $\lambda_{div}$ are all set to 0.1. We train our model with ADAM optimizer~\cite{kingma2014adam}.
The transformation layer $\mathcal{T}$ is implemented using a 2-layer MLP layer with 256 hidden units. 
Experiments were implemented in Pytorch and run on a server with 8 Nvidia GeForce RTX 3090 GPUs.~\footnote{\ Code is available at \url{https://anonymous.4open.science/r/FCS-Submission-LGDESetNet}.}

\begin{figure*}[t]
    \centering
    \includegraphics[trim=2mm 4mm 1mm 2mm, clip=true, width=0.23\linewidth]{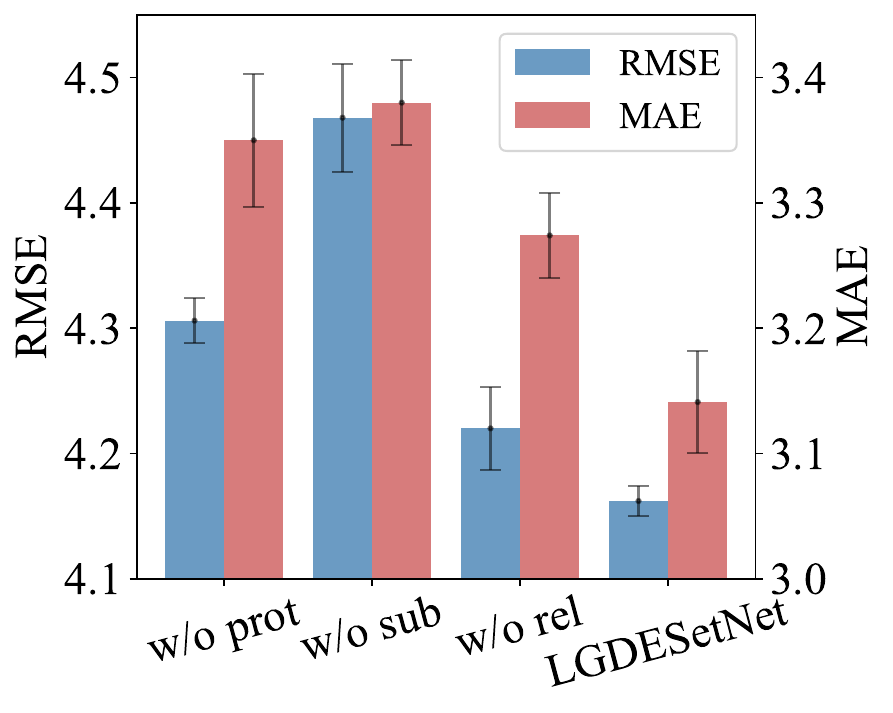}
    \includegraphics[trim=2mm 4mm 1mm 2mm, clip=true, width=0.23\linewidth]{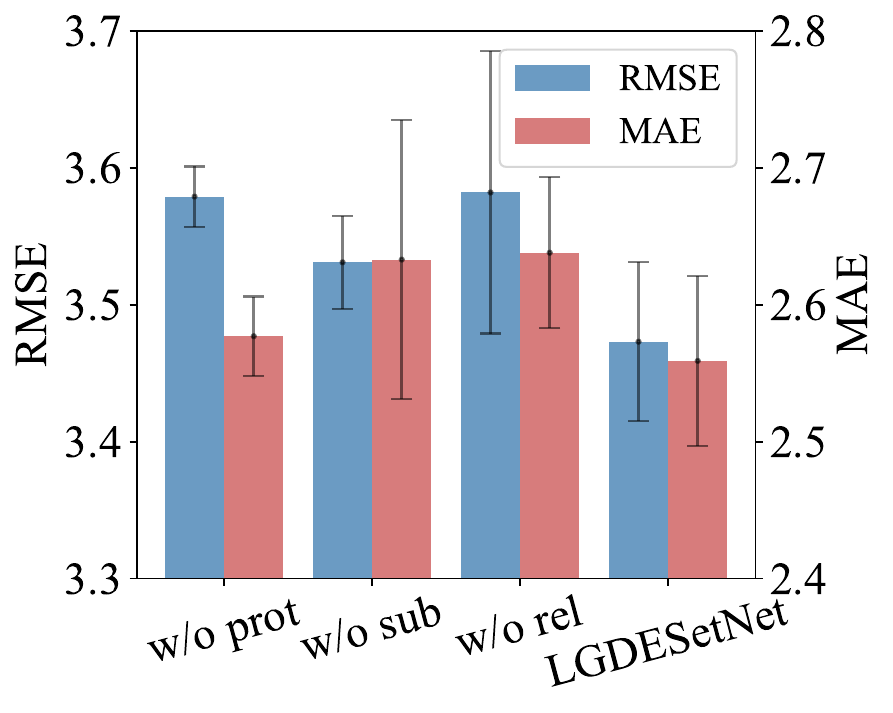}
    \includegraphics[trim=2mm 4mm 1mm 2mm, clip=true, width=0.23\linewidth]{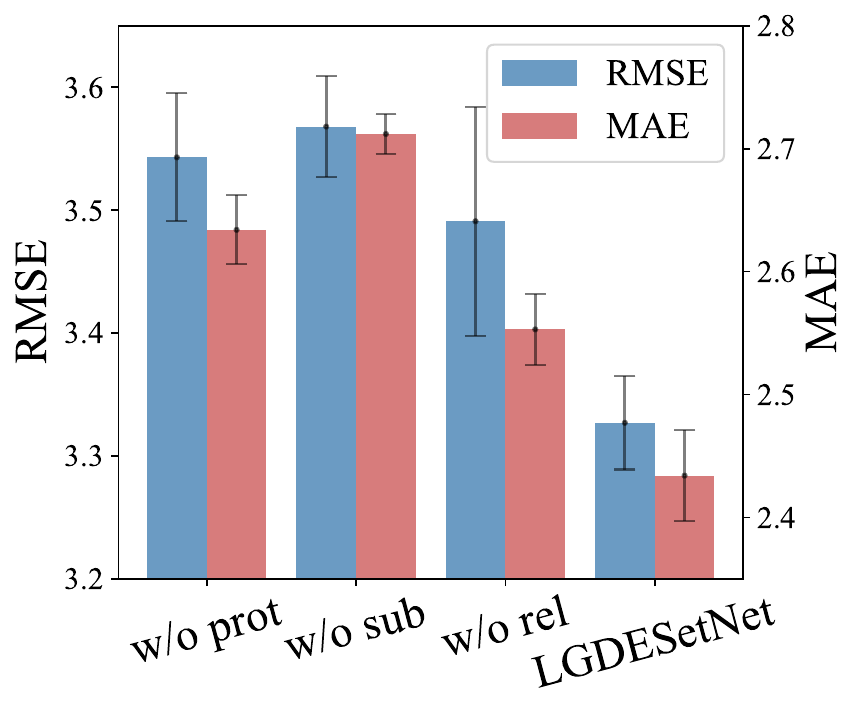}
    \includegraphics[trim=2mm 4mm 1mm 2mm, clip=true, width=0.23\linewidth]{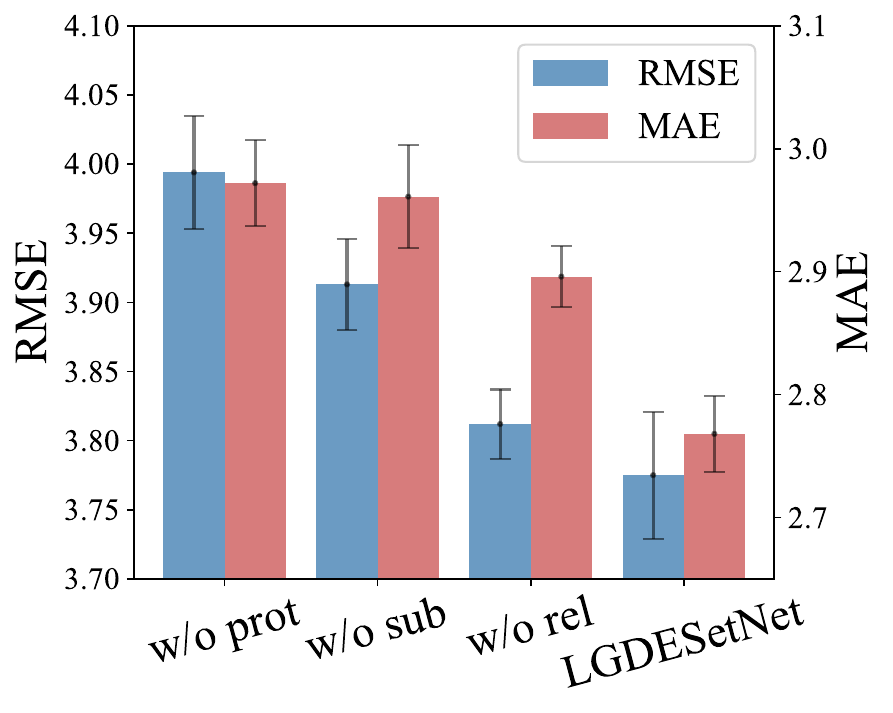}
    
    \caption{Ablation studies across four datasets.}
    \label{fig:ablation}
\end{figure*}

\subsection{Overall Performance Evaluation}
The overall prediction performance of LGDESetNet and all baselines on four datasets is reported in~\Cref{table:lower_performance}. 
We can observe that:
(1) LGDESetNet consistently outperforms existing salary prediction baselines across all datasets, demonstrating that our comprehensive framework more effectively captures the complex factors influencing salary. This validates our hypothesis that a holistic skill set-based reasoning process can effectively model intrinsic salary-influential patterns.
(2) Compared to self-explainable baselines, LGDESetNet achieves an average improvement of 10\%. This substantial gain stems from our model's ability to disentangle and prototype relationships among set entities while maintaining job context awareness, enabling more accurate identification of salary-influential skill patterns.
Unlike existing self-interpretable models that often trade accuracy for explainability, LGDESetNet matches or even exceeds the performance of complex black-box models, as shown in Supplementary~\Cref{table:lower_performance_additional}. This demonstrates its ability to bridge the typical performance gap between interpretable and black-box approaches, offering a more transparent yet accurate solution for set modeling.
(3) The notable underperformance of Set-Tree compared to other methods highlights the importance of our approach. Set-Tree's limitations in capturing job context-dependent nuances in set data and modeling combinatorial interactions among set elements underscore the challenges in explainable skill-based set modeling.

\subsection{Ablation Study}
We introduce three variants of our LGDESetNet by disabling some parts of the network to evaluate the effectiveness of each component: (1) \textbf{w/o prot}, which disables the prototypical set learning module; (2) \textbf{w/o sub}, which disables the disentangled discrete subset selection module; and (3) \textbf{w/o rel}, which utilizes a conventional prototype replacement strategy~\cite{ming2019interpretable} instead of our embedding-projection training algorithm. 
The results are reported in~\Cref{fig:ablation}.
It can be observed that removing any modules from LGDESetNet reduces its performance, as the disentangled subset selection and prototypical learning modules are essential for capturing the full utility of set-valued reasoning. 
Particularly, the removal of the disentangled subset selection module leads to the most significant performance drop, underscoring the importance of skill semantic disentanglement. 
LGDESetNet outperforms these two components, which indicates that it not only raises the model explainability but also can increase the fitting ability by considering semantic interactions between subsets.
Furthermore, LGDESetNet shows a 2\%-5\% performance gain over its variant with a classical projection strategy~\cite{ming2019interpretable}, validating our embedding-projection training's effectiveness in advancing prototype learning.

\begin{figure}[t]
    \centering
    \includegraphics[trim=2mm 4mm 1mm 2mm, clip=true, width=0.45\linewidth]{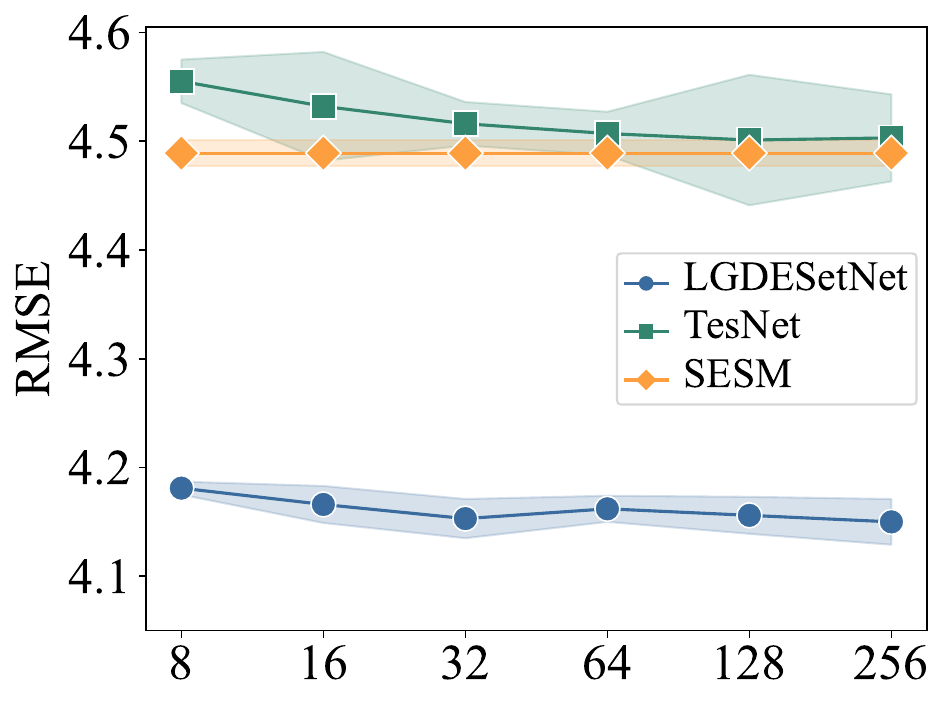}
    \hspace{0.05\linewidth}
    \includegraphics[trim=2mm 4mm 1mm 2mm, clip=true, width=0.45\linewidth]{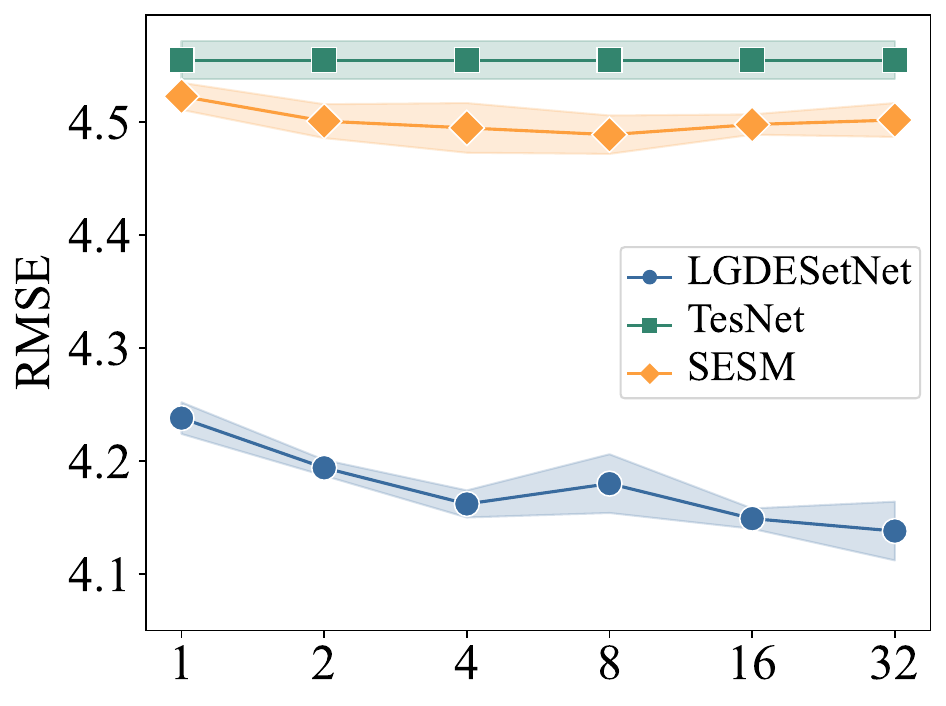}
    \caption{Parameter sensitivity of IT job postings. Left: Prototype number $M$. Right: View number $H$.}
    \label{fig:Params}
\end{figure}

\subsection{Parameters Sensitivity}
We further conduct the hyper-parameter analysis of the prototype number $M$ and the view number $H$. Specifically, we evaluate the performance of LGDESetNet with varying $M$ and $H$ on the IT dataset, compared with instance-wise selection model SESM~\cite{geng2020does} and prototype-based model TesNet~\cite{wang2021interpretable}.
The results are shown in~\Cref{fig:Params}.
With the number of prototypes $M$ increases, the performance gets minorly better and stabilizes within a certain range. 
This pattern suggests that our model effectively captures the intrinsic dimensionality of the skill embedding space with a relatively small number of prototypes. Importantly, LGDESetNet maintains consistent performance advantages over baseline methods across all prototype settings, highlighting its inherent robustness to this hyperparameter. The stability plateau indicates that excessive prototypes may not necessarily improve the model's performance, as they could introduce noise and redundancy.
By default, we choose $M=64$ in order to get more diverse and representative skill prototypes. 

The view number $H$ (Figure~\ref{fig:Params}) exhibits more pronounced sensitivity, with a steep performance improvement from $H=1$ to $H=4$, followed by marginal gains thereafter. This pattern reveals that multiple views are essential for capturing the multi-faceted nature of skill representations, but only up to a certain threshold. The rapid early improvement suggests that different views effectively capture complementary aspects of skill semantics. However, the diminishing returns beyond $H=4$ indicate a saturation point where additional views contribute minimal new information. Moreover, excessive view fragmentation can partition the semantic space too finely, potentially compromising both computational efficiency and interpretability. Based on this analysis, we adopt $H=4$ for our experiments as the optimal balance between model performance and semantic coherence.

\begin{table}[t]
    \centering
    \caption{Results on subset cohesion score.}
    \begin{adjustbox}{max width=\linewidth}
    \begin{tabular}{l|cccc}
    \toprule[0.5pt] \specialrule{0em}{1.5pt}{1.5pt}
    Methods & IT   & Designer    & High-tech & Finance   \\ \midrule \specialrule{0em}{1.5pt}{1.5pt}
    SESM    & $\text{0.085}_{\pm \text{0.012}}$ & $\text{0.035}_{\pm \text{0.018}}$  & $\text{0.095}_{\pm \text{0.025}}$   &         $\text{0.074}_{\pm \text{0.023}}$            \\
    LGDESetNet & $\text{0.126}_{\pm \text{0.034}}$ & $\text{0.113}_{\pm \text{0.021}}$ &  $\text{0.186}_{\pm \text{0.011}}$         &      $\text{0.127}_{\pm \text{0.034}}$ \\ 
    \bottomrule[0.5pt]
    \end{tabular}
    \end{adjustbox}
    \label{tab:label_scs}
\end{table}

\subsection{Quantitative Explanation Analysis}
We evaluate LGDESetNet's ability to identify semantically cohesive skill subsets using the Subset Cohesion Score (SCS). SCS quantifies the semantic tightness of a subset $S$, calculated as $\text{SCS}(S) = \frac{2}{|S|(|S|-1)} \sum_{x_i,x_j \in S} \mathcal{F}(x_i, x_j)$, where $\mathcal{F}$ indicates skill co-occurrence frequency across the dataset. As shown in Table~\ref{tab:label_scs}, LGDESetNet substantially outperforms SESM across all four domains. These results demonstrate that our approach more effectively captures domain-specific skill relationships, with the greatest advantages observed in highly specialized fields where skill coherence is crucial. The consistent performance improvement across diverse domains indicates that LGDESetNet's co-occurrence graph density regularization and global prototype-based selection mechanism successfully identify meaningful skill compositions rather than merely selecting individual skills. 

\begin{figure*}[t]
    \centering

    \begin{minipage}[t]{0.3\linewidth}
        \centering
        \includegraphics[width=\linewidth]{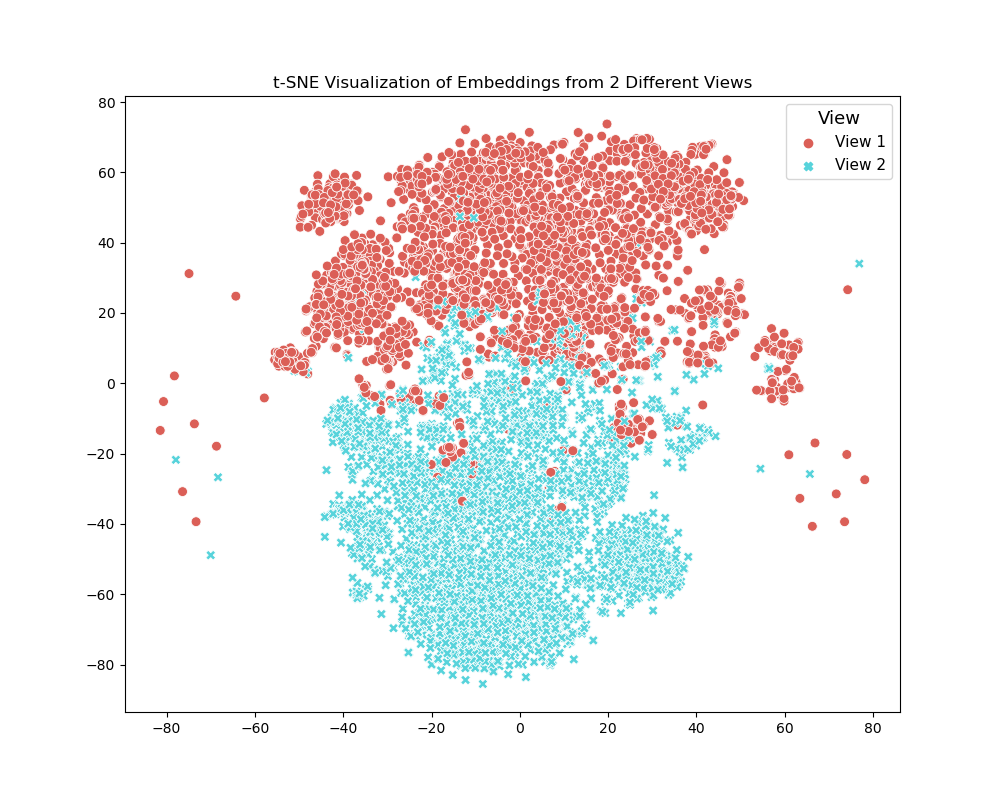}
        \vspace{1mm}
        \small\textbf{(a)} $H = 2$
    \end{minipage}
    \hspace{0.035\linewidth}
    \begin{minipage}[t]{0.3\linewidth}
        \centering
        \includegraphics[width=\linewidth]{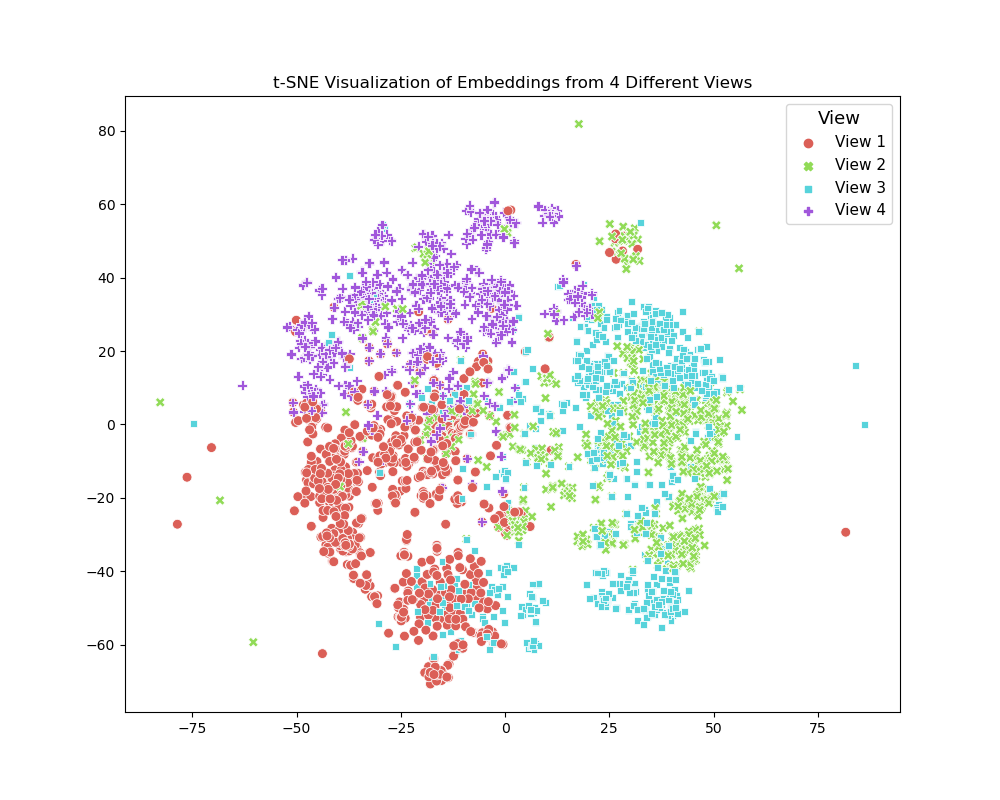}
        \vspace{1mm}
        \small\textbf{(b)} $H = 4$
    \end{minipage}
    \hspace{0.035\linewidth}
    \begin{minipage}[t]{0.3\linewidth}
        \centering
        \includegraphics[width=\linewidth]{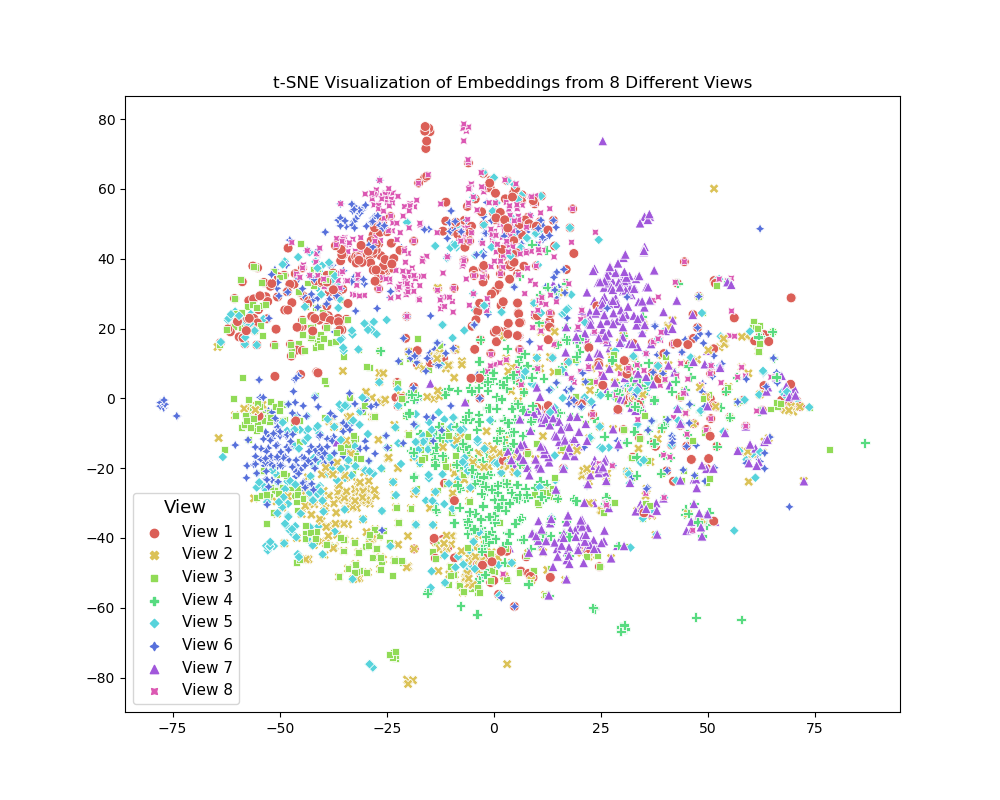}
        \vspace{1mm}
        \small\textbf{(c)} $H = 8$
    \end{minipage}

    \caption{t-SNE visualizations of skill set embeddings under different numbers of views $H$.}
    \label{fig:tsne-comparison}
\end{figure*}

\subsection{View Visualization}
\Cref{fig:tsne-comparison} shows the t-SNE visualization of extracted skill subsets with different view numbers $H$. With $H$ = 2 or 4, we can observe that the skill embeddings are more compact and well-separated. Some overlap between views is reasonable, as certain skills share common traits, but the overall semantic meaning remains distinct. This indicates that LGDESetNet effectively captures the intrinsic skill relationships. When expanded to $H$ = 8, the skill embeddings become more dispersed, with less clear boundaries between different skill clusters. This suggests that excessive views may introduce redundancy, potentially compromising the model's interpretability. This is also consistent with our previous experimental findings (shown in \Cref{fig:Params}) that $H$ = 4 is the optimal balance.

\begin{figure}[t]
    \centering

    \includegraphics[width=0.95\linewidth]{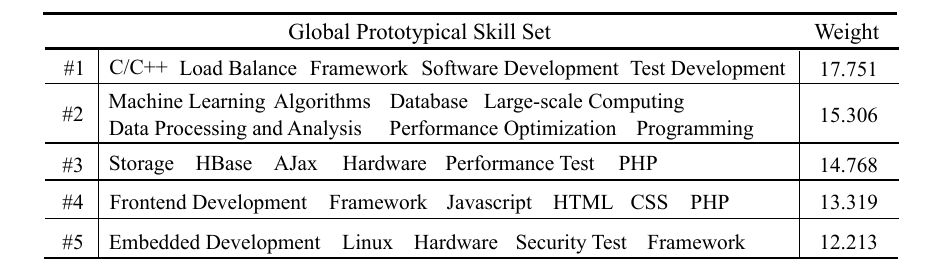}
    \caption*{\textbf{(a)} Top-5 Salary-Influential Prototypical Skill Sets}

    \vspace{2mm}

    \begin{minipage}[b]{0.48\linewidth}
        \centering
        \includegraphics[width=\linewidth]{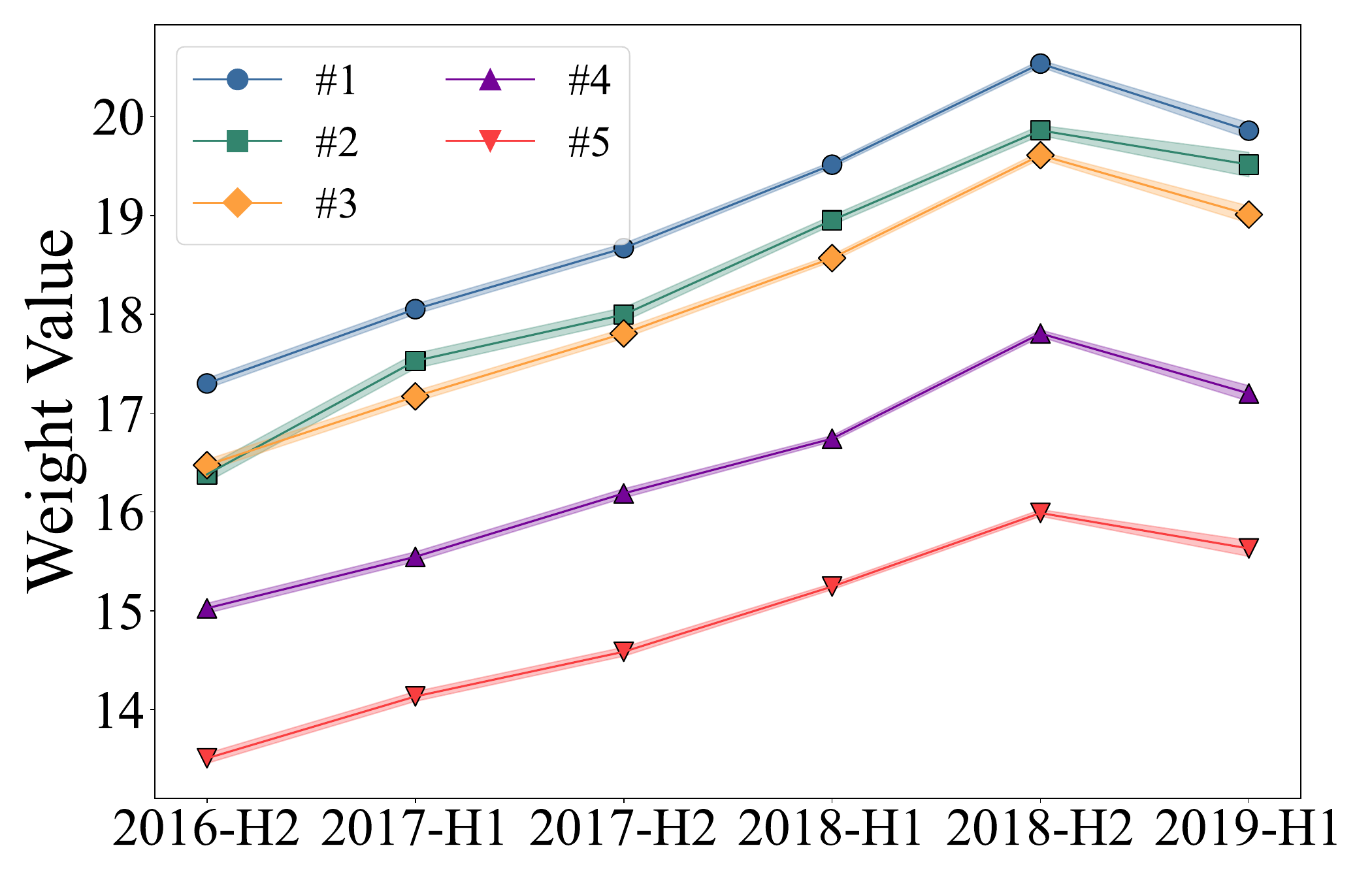}
        \caption*{\textbf{(b)} Time}
    \end{minipage}
    \hfill
    \begin{minipage}[b]{0.48\linewidth}
        \centering
        \includegraphics[width=\linewidth]{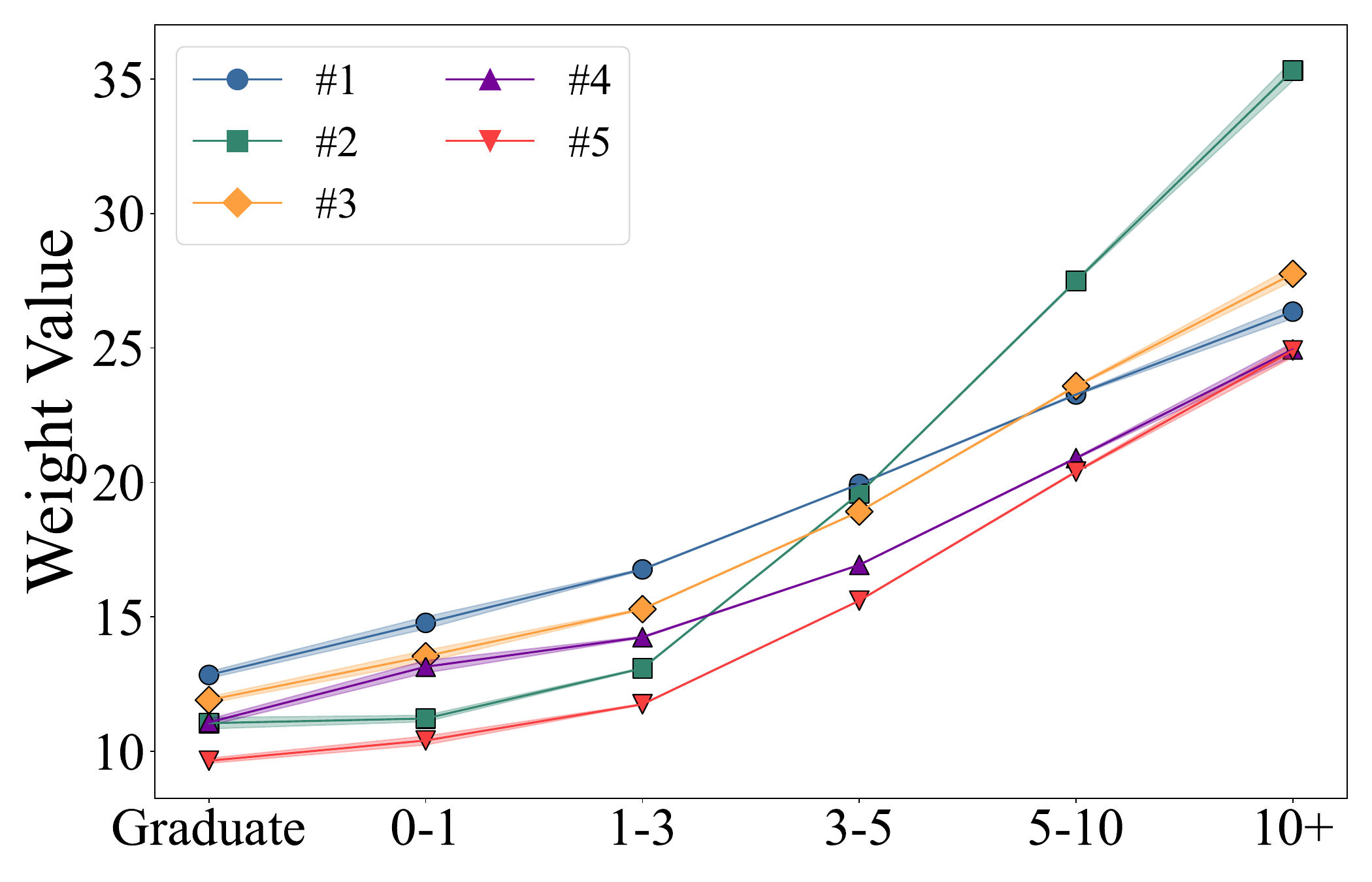}
        \caption*{\textbf{(c)} Work Experience}
    \end{minipage}

    \vspace{2mm}

    \includegraphics[width=0.95\linewidth]{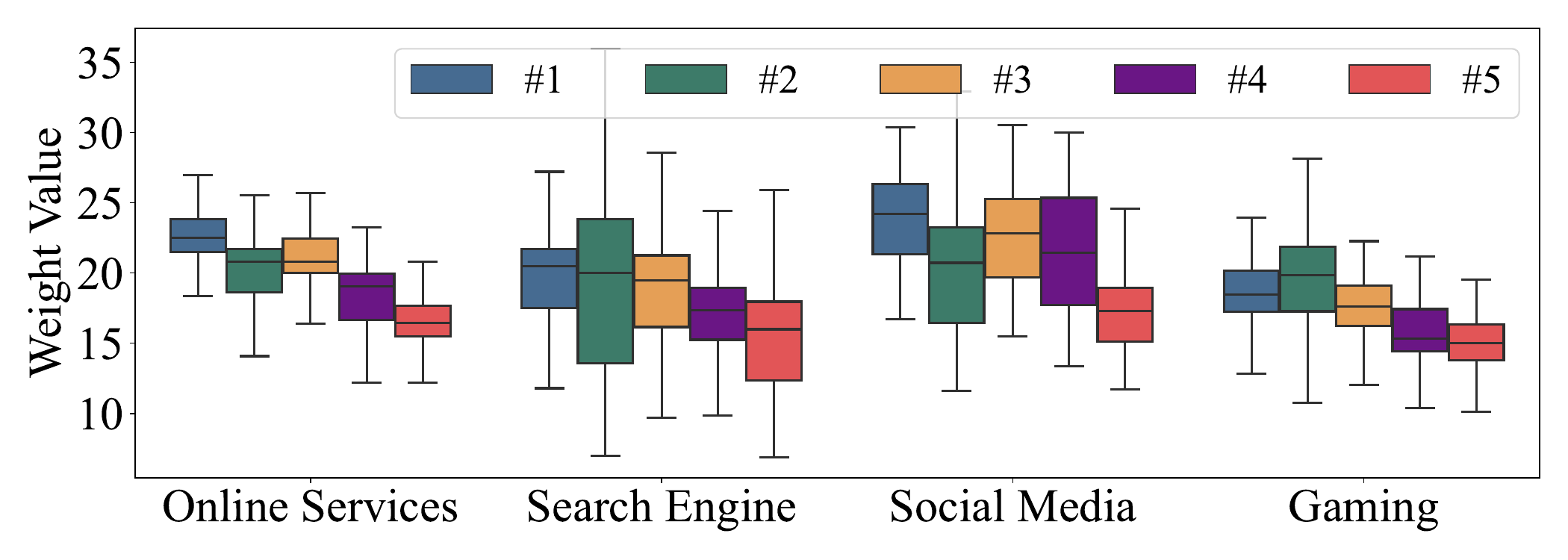}
    \caption*{\textbf{(d)} Service Business}

    \caption{Global salary-influential prototypical skill sets concerning different job contexts.}
    \label{fig:global_vis}
\end{figure}

\begin{figure*}[t]
    \begin{center}
        \includegraphics[width=\linewidth]{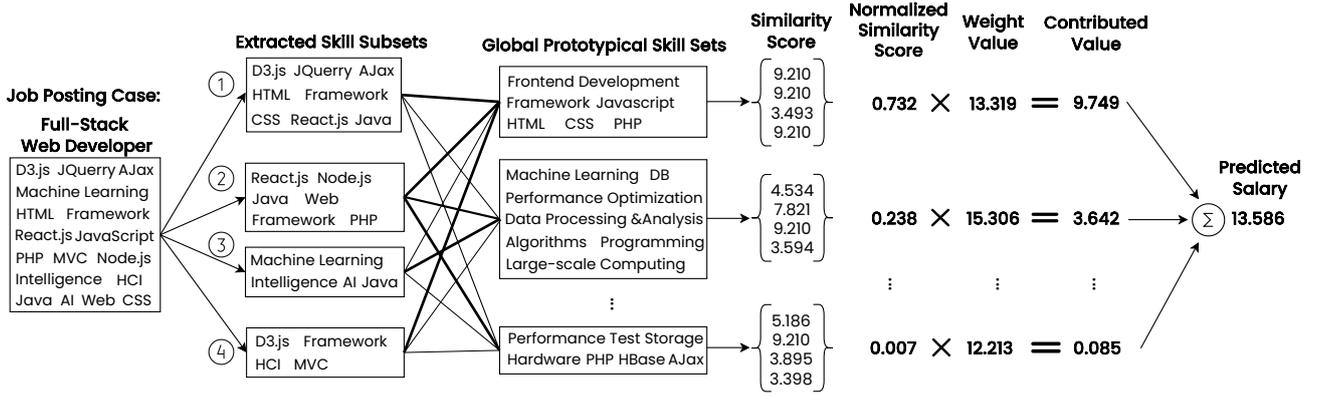}
        \caption{Reasoning processes of LGDESetNet on the IT-related job posting. LGDESetNet identifies semantically meaningful and disentangled subsets from the input set and thereafter provides representative subsets for evidence comparison. 
        }
        \label{fig:reasoning}
    \end{center}
\end{figure*}

\subsection{Salary-influential Skill Set Analysis}
\subsubsection{Representative Skill Sets}
As with our design, LGDESetNet models the combinatorial market-aware salary influences of skill compositions by learning the global prototypical skill sets.
\Cref{fig:global_vis} (a) showcases IT-related skill sets with the top-5 highest average weights $\gamma^{\text{sal}}$ across the dataset.
It illustrates LGDESetNet's ability to highlight diverse job skills, ranging from software to hardware development.
Furthermore, skill sets' weight values indicate their market value for salary prediction, with AI and programming skills valued more than front-end skills, aligning with current labor market trends.
By quantitative comparison with representative skill sets, talents can explicitly capture the labor market demands, identify skills gaps by comparison, and reasonably plan careers.

\subsubsection{Skill Set Influence Concerning Different Job Contexts}
We delve into the dynamic trends of prototypical skill sets across various job contexts, including time, work experience, and service business. 
By observing the fluctuations of prototype weight value, we can achieve many interesting insights into the labor market trends.
From \Cref{fig:global_vis} (b), it's evident that prototypical skill sets show an upward trend, highlighting an increasing demand for IT talents. Interestingly, there is a noticeable downturn in the first half of 2019, which is consistent with the global economic downturn's impact.
\Cref{fig:global_vis} (c) illustrates how work experience influences salary dynamics within job skill sets.
For graduates, development-related skills could bring a higher initial salary. However, as work experience accumulates, the growth in salary slows down, hinting at a potential career ceiling. In contrast, algorithm-related skills demonstrate a more stable, enduring effect on salary progression.
We also investigate the dynamic influences of service industry shown in \Cref{fig:global_vis} (d). %
For example, data analysis and backend development commanding higher median salaries due to their critical roles in managing and interpreting data, while hardware development is often underestimated by IT companies.

\subsection{Case study}

In \Cref{fig:reasoning}, we qualitatively visualize LGDESetNet's transparent reasoning process through a case study on a full-stack web development job posting. 
Initially, LGDESetNet identifies essential skill subsets from varied semantic angles—frontend, backend, machine learning, and interaction design.
These extracted skill subsets are further compared with the global prototypical skill sets to calculate the similarities, which represents the corresponding semantic relevance. 
For example, the prototypical skill set \{\emph{``Frontend Development''}, \emph{``Framework''}, \emph{``JavaScript''}, \emph{``HTML''}, \emph{``CSS''}, \emph{``PHP''}\} has high similarity scores with the extracted subset 1, 2, and 4, which implies the high relevance of \emph{``Frontend Development''} skills in the job posting.
In the generalized additive aggregation module, these similarity scores are weighted and summed together to get a final predicted salary value. 
We can observe that LGDESetNet can highlight the most salary-influential patterns of skill sets. 
The skill topic \emph{``Frontend Development''} primarily aligns with the input set that has the highest similarity scores. Moreover, it significantly influences salary prediction. 
While IT-related skills such as \emph{``Algorithms''} and \emph{``Data Processing \& Analysis''}, can bring a high potential salary (high $\gamma^\text{sal}$), they have low relevance to web development and thereafter have fewer contributions to the outcome. 
This case study highlights LGDESetNet's ability to clarify the salary reasoning process, with potential applications for various audiences. Recruiters may set equitable salaries, job seekers align skills with market trends, and policymakers inform education and workforce policies.

\begin{table}[t]
    \centering
    \caption{{User study results: 5-point Likert scale ratings of three
    measured features for explanatory methods.}}
    \begin{adjustbox}{max width=\linewidth}
    \begin{tabular}{l|ccc}
    \toprule[1.0pt]
    Methods         & Black box & ProtoPNet & LGDESetNet \\ \midrule
    Understanding   &       3.742         &       3.927             &   \textbf{4.421}         \\
    Trust &     3.385          &     3.521               &  \textbf{4.220} \\
    Usability &     3.046            &     4.028               &  \textbf{4.623}         \\ \bottomrule[1.0pt]
    \end{tabular}
    \end{adjustbox}
    \label{tab:user_study}
\end{table}

\subsection{User Study}
Following the design rationale of~\cite{rong2023towards}, we conduct a user study to evaluate LGDESetNet's explainability by accessing the user's comprehension of the model's reasoning process (Understanding), their confidence and comfort with the recommendations (Trust), and the intuitiveness and ease of use of provided explanations (Usability).
Specifically, we recruited 53 postgraduate students and staffs with the background of computer science or engineering from different universities for evaluation.
We compare our network to a full black box (HSBMF), where people are simply told ``a predicted result produced by a trained machine learning model'' in text form, and a representative prototyping model (ProtoPNet), where the learned prototypical skill sets are presented as explanations. 
For fair comparison, all the designed questions stem from each model's inherent modeling procedure. As shown in~\Cref{tab:user_study}, We can observe that self-explainable models increase user confidence by providing the evidence of model inference. It shows the necessity of our efforts to qualify the impacts of job skills. Moreover, traditional prototyping methods fail to provide semantically dense skill sets, which might confuse and confound users. In contrast, LGDESetNet can enhance user experience by offering clear, separate skill compositions and using prototypical examples as comparative evidence.

\vspace{-1mm}
\section{Conclusion}

In this study, we addressed the challenge of understanding how job skills
influence job salaries.
We introduced a novel intrinsically explainable set-based neural prototyping approach, LGDESetNet, to assess the impact of skill compositions on job salary, offering insights into disentangled skill sets influencing salary from a local and global viewpoint. Specifically, our approach leveraged a disentangled discrete subset selection module to pinpoint multi-faceted influential subsets with diverse semantics. 
In addition, we proposed a prototypical set learning method to distill globally influential skill sets. The final output transparently showcased the semantic relationships between these subsets and prototypes. 
Extensive evaluations on real-world datasets highlighted the effectiveness and explainability of LGDESetNet as a pioneering self-explainable neural methodology for processing skill-based salary prediction.

\begin{acknowledgement}
  This work is partly supported by the National Key Research and Development Program of China (No. 2023YFF0725001), the National Natural Science Foundation of China (No. 62306255, 92370204, 62176014), the Natural Science Foundation of Guangdong Province (No. 2024A1515011839), the Fundamental Research Project of Guangzhou (No. 2024A04J4233), the Guangzhou-HKUST(GZ) Joint Funding Program (No.2023A03J0008), and the Education Bureau of Guangzhou Municipality. 
\end{acknowledgement}

\section*{Appendixes}
\addcontentsline{toc}{section}{Appendixes} %
\renewcommand{\thesection}{Appendix \Alph{section}}
\setcounter{section}{0}

\section{Proof of Permutation Invariance of LGDESetNet}
\label{appendix:proof}

\begin{lemma}(Permutation matrix and permutation function)
$\forall X \in \mathbb{R}^{N \times N}$, for all permutation matrix $P$ of size $N$, there exists $\pi: \{1,2,...,N\} \rightarrow \{1,2,...,N\}$ is a permutation function, which satisfies:
$$X_{ij} = (P X)_{\pi(i)j} = (X  P^T)_{i\pi(j)} = (P X  P^T)_{\pi(i)\pi(j)}$$
\end{lemma}

\begin{lemma}
Given $X \in \mathbb{R}^{N \times N}$ and any permutation matrix $P$ of size $N$, we have 
$$\text{Sigmoid}(P X  P^T) = P  \text{Sigmoid}(X)  P^T.$$
\end{lemma}

\begin{proof}
We have $\text{Sigmoid}(X)_{ij} = \frac{1}{1 + e^{-X_{ij}}}$ and now consider the permutation function $\pi$ for $P$. Following the lemma 1, we can get:
\begin{equation*}
   \begin{split}
       (P  \text{Sigmoid}(X)  P^T)_{\pi(i)\pi(j)} & = \text{Sigmoid}(X)_{ij} \\
                                                        & = \frac{1}{1 + e^{-X_{ij}}} \\
                                                        & = \frac{1}{1 + e^{-(P X  P^T)_{\pi(i)\pi(j)}}} \\
                                                        & = \text{Sigmoid}(P  X  P^T)_{\pi(i)\pi(j)}.
   \end{split}
\end{equation*}
which implies $P  \text{Sigmoid}(X)  P^T = \text{Sigmoid}(P  X  P^T)$.
\end{proof}

As presented before, we inject the Gumbel noises into the Sigmoid function for enabling a differentiable set element sampling process. Given a weight matrix $X$, this Gumbel-Sigmoid Sampler (GSS) can be formulated as:
\begin{equation*}
    \begin{split}
        \operatorname{GSS}(X) = \frac{\exp((X + G_0)/\tau)}{\exp((X + G_0)/\tau) + \exp((G_1)/\tau)}.
    \end{split}
\end{equation*}

\begin{lemma}
Given $X \in \mathbb{R}^{N \times N}$ and any permutation matrix $P$ of size $N$, we have 
$$\operatorname{GSS}(PXP^T) = P \operatorname{GSS}(X) P^T.$$
\begin{proof}
Considering the GSS function only adds element-wise operations on Sigmoid, it does not change the original permutation property of Sigmoid. 
\end{proof}
\end{lemma}

\begin{figure*}[t]
    \centering

    \begin{minipage}[t]{0.23\linewidth}
        \centering
        \includegraphics[width=\linewidth]{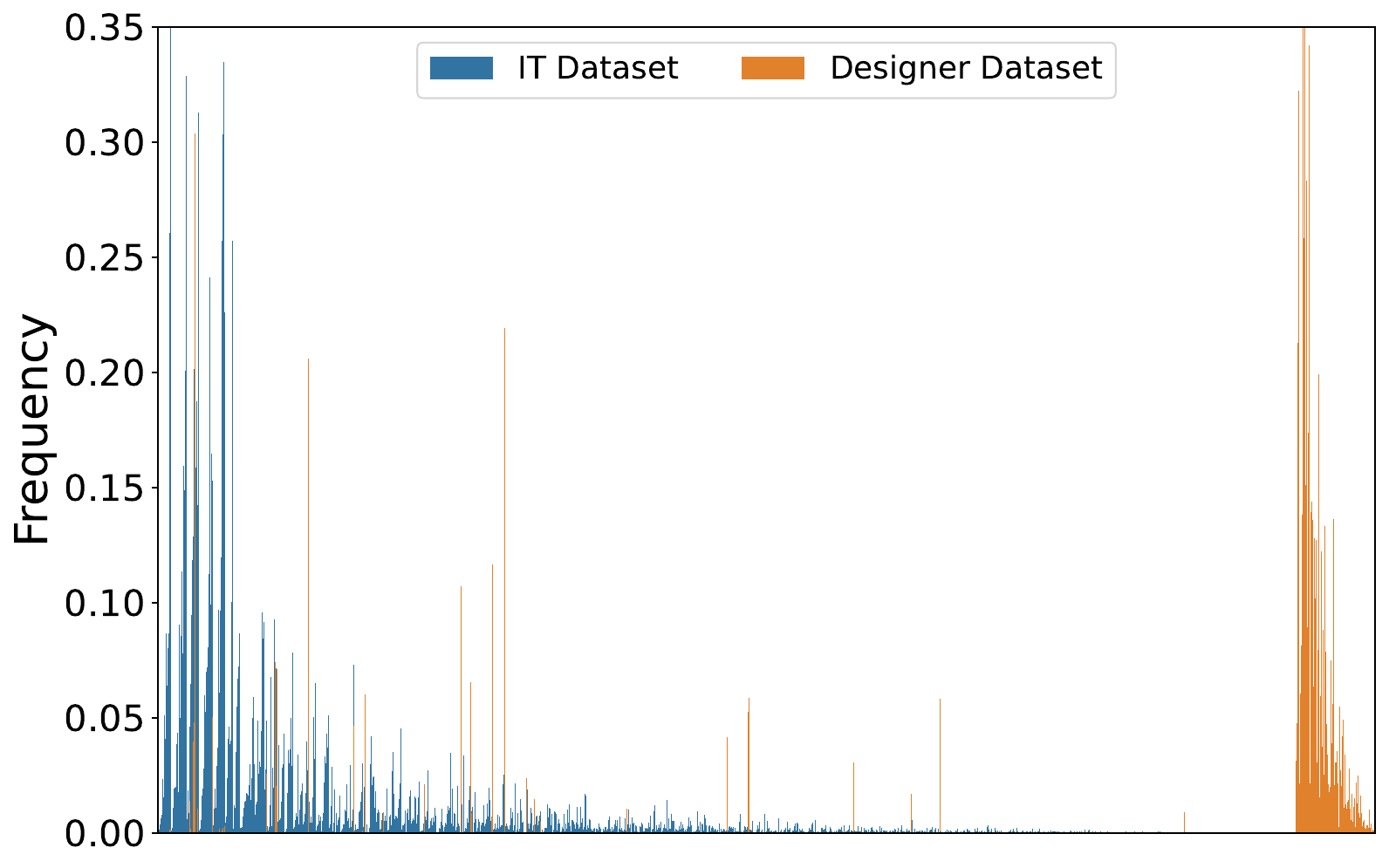}
        
        \small\textbf{(a)} Skill distribution
    \end{minipage}
    \hspace{0.01\linewidth}
    \begin{minipage}[t]{0.23\linewidth}
        \centering
        \includegraphics[width=\linewidth]{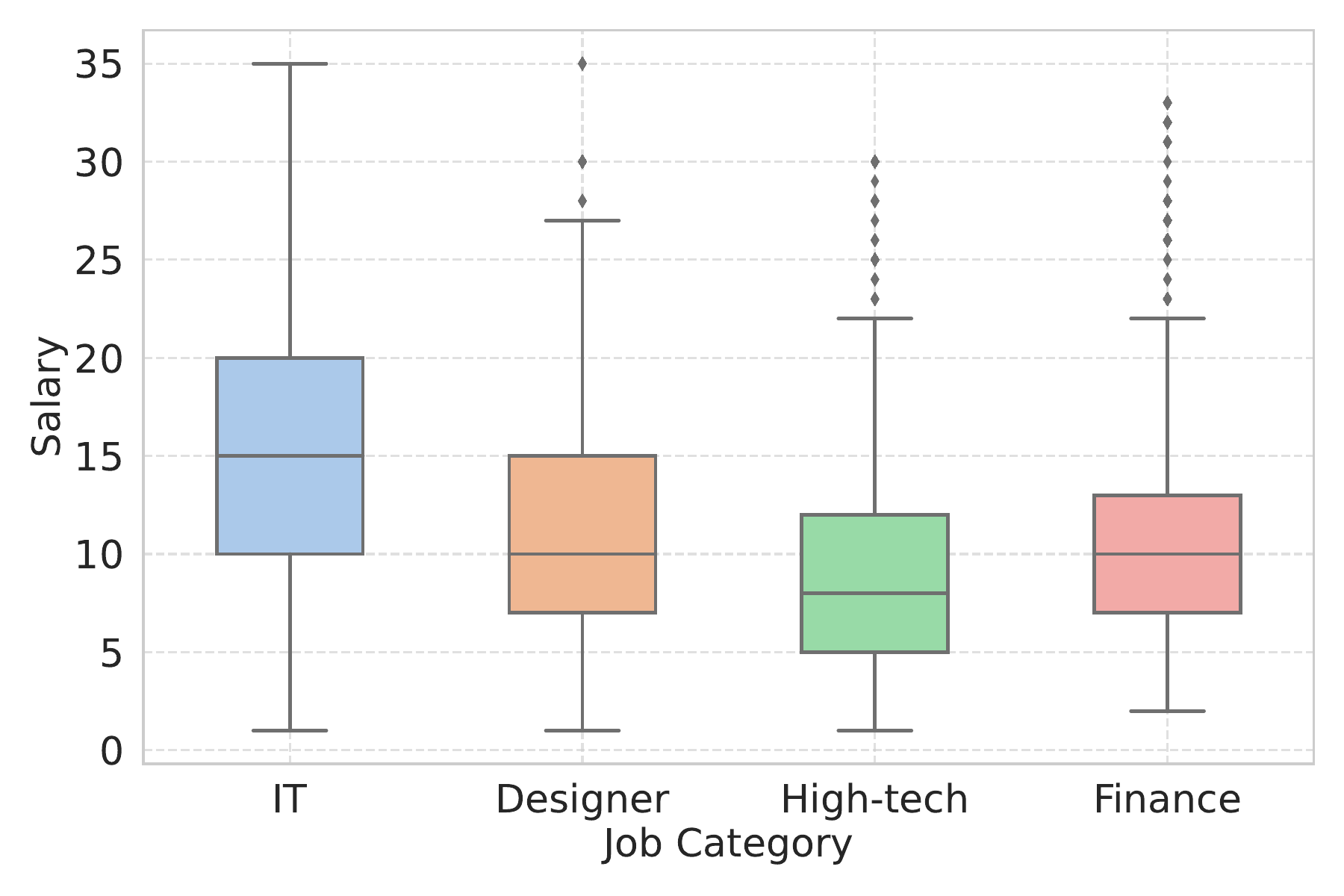}

        \small\textbf{(b)} Salary distribution
    \end{minipage}
    \hspace{0.01\linewidth}
    \begin{minipage}[t]{0.23\linewidth}
        \centering
        \includegraphics[width=\linewidth]{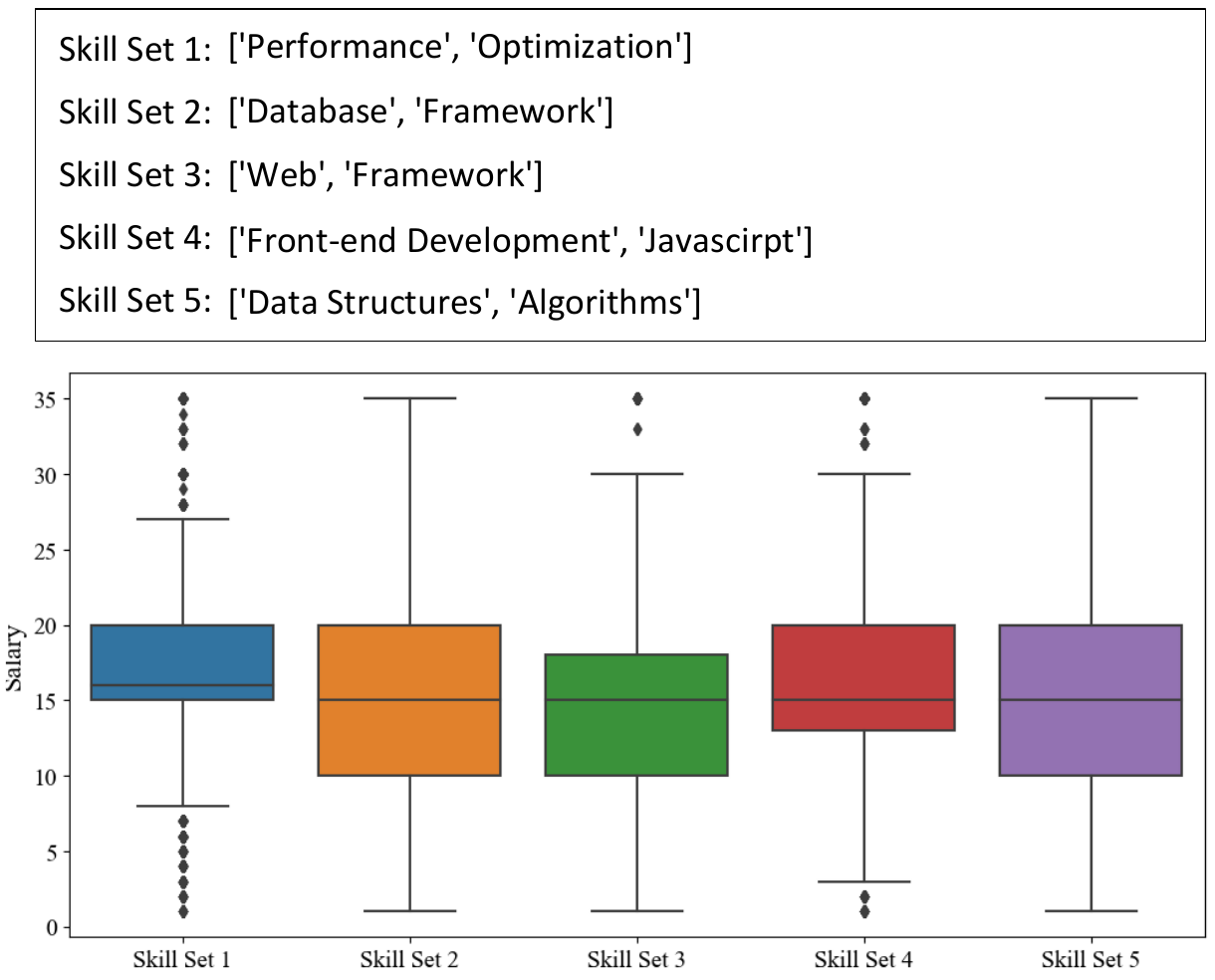}

        \small\textbf{(c)} Frequency analysis
    \end{minipage}
    \hspace{0.01\linewidth}
    \begin{minipage}[t]{0.23\linewidth}
        \centering
        \includegraphics[width=\linewidth]{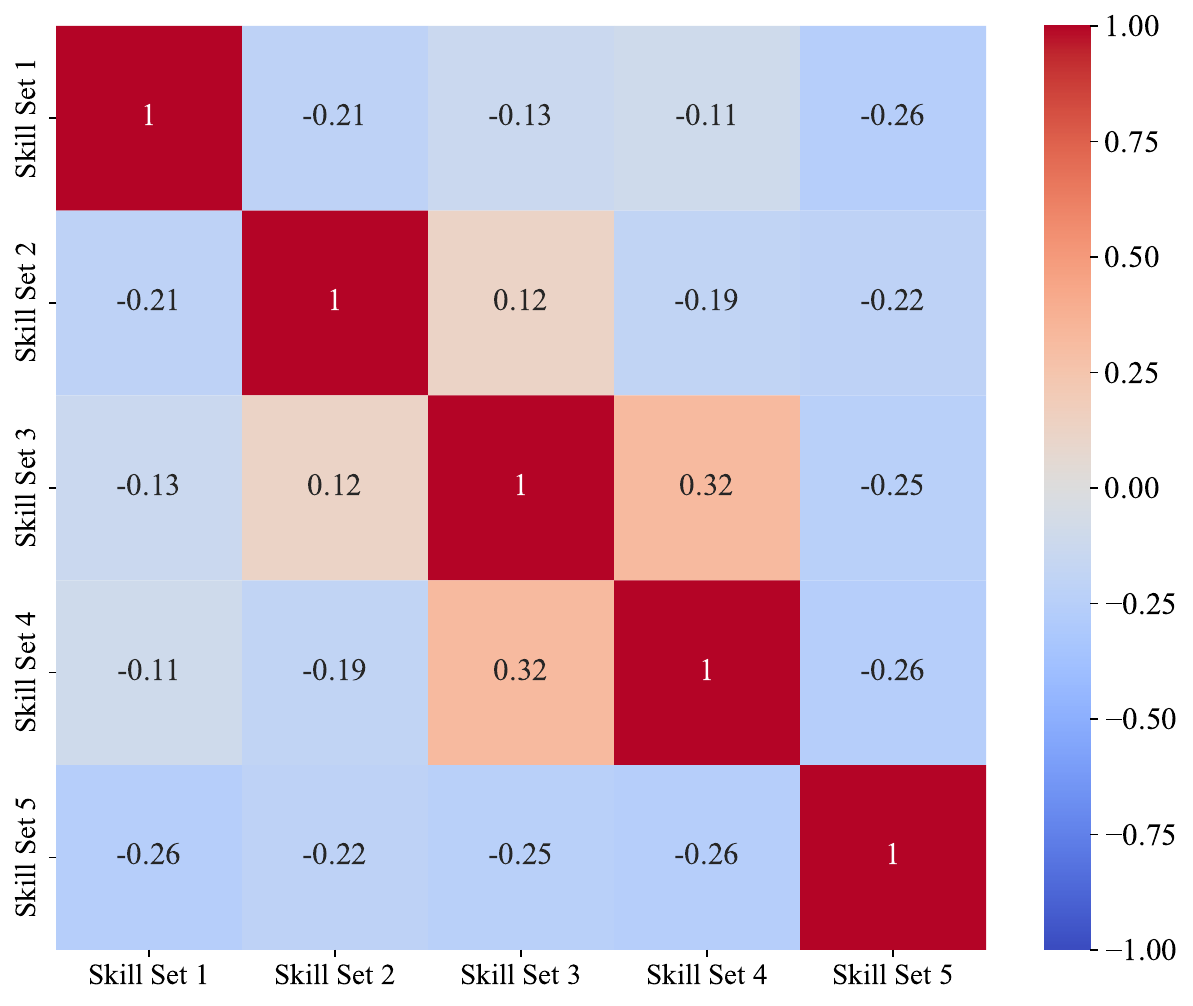}

        \small\textbf{(d)} Spearman correlation
    \end{minipage}

    \caption{Dataset statistics and analysis of frequent skill sets (minSup $\geq$ 0.01).}
    \label{fig:stat}
\end{figure*}

\begin{figure*}[htbp]
    \begin{center}
        \includegraphics[width=0.8\linewidth]{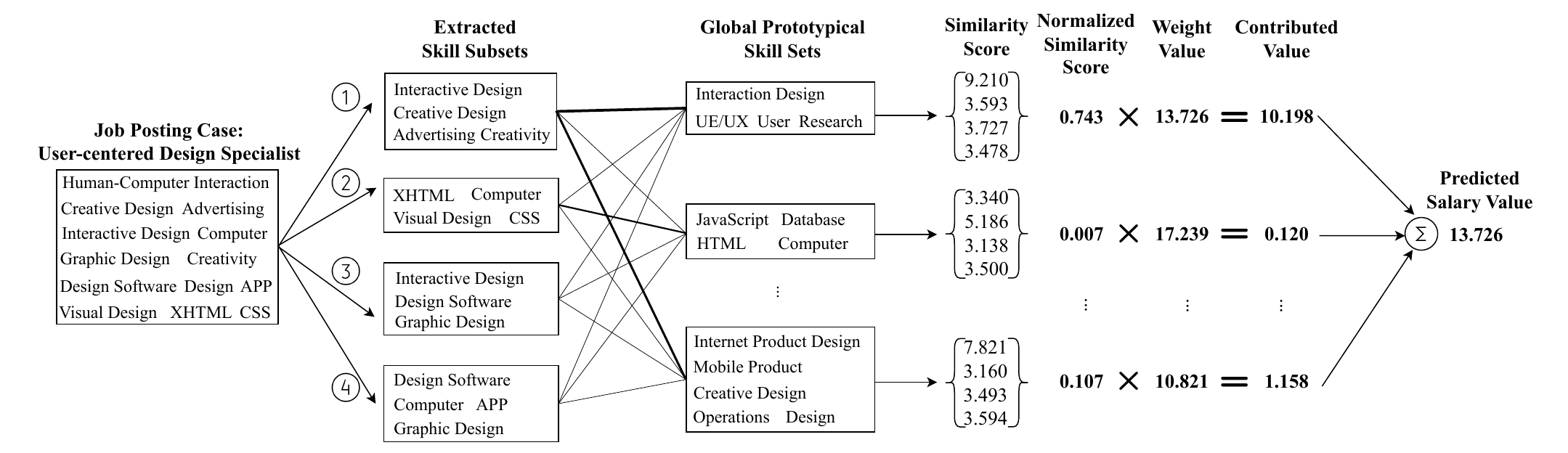}
        \caption{Reasoning process of LGDESetNet on Designer job posting.}
        \label{fig:additional-case-study}
        \vspace{-4mm}
    \end{center}
\end{figure*}

\begin{table*}[t]
    \centering
    \caption{Additional performance evaluation of SOTA set-based black-box models. LGDESetNet achieves comparable performance to the best-performing black-box baselines, with slight improvements. The best results are highlighted in bold.}
    \resizebox{\linewidth}{!}{%
    \begin{tabular}{l|cccccccc}
    \toprule[0.5pt]
    \multirow{4}{*}{Methods}
    & \multicolumn{4}{c}{Skill-Salary Job Postings~\cite{sun2021market}} & \multicolumn{4}{c}{Job Salary Benchmarking~\cite{meng2022fine}} \\
    \cmidrule(r){2-5} \cmidrule(r){6-9}
    & \multicolumn{2}{c}{IT} & \multicolumn{2}{c}{Designer} & \multicolumn{2}{c}{High-tech} & \multicolumn{2}{c}{Finance} \\ \cmidrule(r){2-9}
    & RMSE ($\downarrow$) & MAE ($\downarrow$) & RMSE ($\downarrow$) & MAE ($\downarrow$) & RMSE ($\downarrow$) & MAE ($\downarrow$) & RMSE ($\downarrow$) & MAE ($\downarrow$) \\
    \midrule
    $\text{DeepSets}$ & $\text{4.271}_{\pm \text{0.020}}$ & {$\text{3.260}_{\pm \text{0.021}}$} & {$\text{3.504}_{\pm \text{0.058}}$} & {$\text{2.583}_{\pm \text{0.033}}$} & {$\text{3.361}_{\pm \text{0.064}}$} & {$\text{2.480}_{\pm \text{0.051}}$} & {$\text{3.835}_{\pm \text{0.075}}$} & $\text{2.859}_{\pm \text{0.067}}$ \\
    Set Transformer & {$\text{4.245}_{\pm \text{0.048}}$} & $\text{3.278}_{\pm \text{0.024}}$ & $\text{3.627}_{\pm \text{0.166}}$ & $\text{2.702}_{\pm \text{0.181}}$ & $\text{3.440}_{\pm \text{0.065}}$ & $\text{2.563}_{\pm \text{0.063}}$ & $\text{3.873}_{\pm \text{0.103}}$ & {$\text{2.883}_{\pm \text{0.091}}$} \\ 
    SetNorm & {$\text{4.231}_{\pm \text{0.057}}$} & $\text{3.262}_{\pm \text{0.023}}$ & $\text{3.533}_{\pm \text{0.042}}$ & $\text{2.627}_{\pm \text{0.028}}$ & $\text{3.414}_{\pm \text{0.049}}$ & $\text{2.538}_{\pm \text{0.028}}$ & $\text{3.874}_{\pm \text{0.036}}$ & {$\text{2.799}_{\pm \text{0.030}}$} \\ \midrule
    LGDESetNet & \textbf{${\text{4.162}_{\pm \text{0.012}}}$} & \textbf{${\text{3.141}_{\pm \text{0.041}}}$} & \textbf{${\text{3.473}_{\pm \text{0.058}}}$} & \textbf{${\text{2.559}_{\pm \text{0.062}}}$} & \textbf{${\text{3.327}_{\pm \text{0.038}}}$} & \textbf{${\text{2.434}_{\pm \text{0.037}}}$} & \textbf{${\text{3.775}_{\pm \text{0.046}}}$} & \textbf{${\text{2.768}_{\pm \text{0.031}}}$} \\
    \bottomrule[0.5pt]
    \end{tabular}}
    \label{table:lower_performance_additional}
\end{table*}

\setcounter{prop}{0}
\begin{prop}
    Given the input skill vector $V_X$ in any permutation order of the skill set $X$, the $h$-$th$ view of the multi-view subset selection network $\operatorname{MSSN}_h(V_X) = \text{pool}\left(\operatorname{GSS}\left(V_h^\text{Q}(V_h^\text{K})^T/\sqrt{d}\right) \cdot V^\text{V}_h\right)$ is a permutation-invariant function.
\begin{proof}
For the dot-product attention mechanism~\cite{vaswani2017attention}, we have $V_h^\text{Q} = V_X W^\text{Q}_h$, $V_h^\text{K} = V_X W^\text{K}_h$, and $V_h^\text{V} = V_X W^\text{V}_h$. Therefore, we can get
\begin{equation*}
    \begin{split}
        & \operatorname{MSSN}_h(P V_X) \\ 
            & = \text{pool}\left(\operatorname{GSS}(V_h^\text{Q}(V_h^\text{K})^T/\sqrt{d}) V_h^\text{V}\right) \\
            & = \text{pool}\left(\operatorname{GSS}(P V_X W^\text{Q}_h (P V_X W^\text{K}_h)^T /\sqrt{d}) P V_X W^\text{V}_h\right) \\
            & = \text{pool}\left(\operatorname{GSS}(P (V_X W^\text{Q}_h) (V_X W^\text{K}_h)^T P^T /\sqrt{d}) P V_X W^\text{V}_h\right) \\
            & = \text{pool}\left(P \operatorname{GSS}((V_X W^\text{Q}_h) (V_X W^\text{K}_h)^T/\sqrt{d}) P^T P (V_X W^\text{V}_h)\right) \\
            & = \text{pool}\left(P \operatorname{GSS}((V_X W^\text{Q}_h) (V_X W^\text{K}_h)^T/\sqrt{d}) (V_X W^\text{V}_h)\right) \\
            & = \operatorname{MSSN}_h(V_X)
    \end{split}
\end{equation*}
which implies the multi-view subset selection network is permutation-invariant by reason that the pooling function satisfies the property of permutation invariance.
\end{proof}
\end{prop}

\begin{prop}
LGDESetNet satisfies the property of permutation invariance, that means given the input skill vector $V_X \in \mathbb{R}^n$ for any permutation matrix $P$ of size $n$, we can always have
$$\operatorname{LGDESetNet}(P  V_X) = \operatorname{LGDESetNet}(V_X)$$

\begin{proof}
Note that in LGDESetNet, we append a prototypical part network to the multi-view subset selection network. This network compares the extracted subsets with prototypes for weighted similarity aggregation. Since the prototyping operation is directly added to each subset embeddings $Z^\text{s}_h = \mathcal{T}\left(E^{\text{s}}_{h}\right)$, it does not change the permutation property of the multi-view subset selection network. Consequently, we can get that our LGDESetNet satisfies the permutation invariance condition. 
\end{proof}
\end{prop}

\section{Additional Experiments}

\subsection{Detailed Dataset Characteristics}
\label{appendix:dataset}
To better understand the dataset characteristics, we provide detailed dataset statistics and analysis of frequent skill sets in~\Cref{fig:stat}. It describes the skill and salary distributions of all four datasets. They present the unique characteristics of each dataset, such as the skill distribution and salary range. Moreover, we analyze the frequent skill sets (minSup $\geq$ 0.01) in the IT job posting dataset and present the salary distribution and Spearman correlation. The results show that the salary is significantly impacted by the skill sets, and the Spearman correlation indicates the co-appearance relationship of skill sets in job postings.

\subsection{Additional Case Study}
\label{appendix:case-study}
As presented in~\Cref*{fig:additional-case-study}, we provide an additional case study using a Designer-related job posting (e.g., User-centered Design Specialist).  It highlights how LGDESetNet can discern the most salary-impacting skills, noting that \emph{``User Research \& Interactivity''} aligns best with high similarity scores, significantly affecting salary predictions. Conversely, IT-related skills like  \emph{``JavaScript''} and  \emph{``Database''}, while potentially valuable, show low relevance to user-centered design roles and thus contribute less to salary estimations in that context. The study showcases LGDESetNet's direct image of the whole salary reasoning process, aiding individuals in understanding skill gaps for career advancement.

\bibliographystyle{fcs}
\bibliography{ref}

\end{document}